\documentclass{article}

\PassOptionsToPackage{numbers,compress,square}{natbib}
\usepackage[final]{neurips_2021}

\usepackage[dvipsnames,table,xcdraw]{xcolor}
\usepackage{hyperref}
\hypersetup{colorlinks,linkcolor={MidnightBlue},citecolor={MidnightBlue},urlcolor={MidnightBlue}}  

\usepackage[utf8]{inputenc} % allow utf-8 input
\usepackage[T1]{fontenc}    % use 8-bit T1 fonts
\usepackage{url}            % simple URL typesetting
\usepackage{booktabs}       % professional-quality tables
\usepackage{amsfonts}       % blackboard math symbols
\usepackage{nicefrac}       % compact symbols for 1/2, etc.
\usepackage{microtype}      % microtypography

\usepackage{amsmath}
\usepackage{amsthm}
\usepackage{amssymb}
\usepackage{graphicx}
\usepackage{enumerate}
\usepackage{caption}
\usepackage{subcaption}

\usepackage[nameinlink,capitalize]{cleveref}

\usepackage{algorithm}
\usepackage{algorithmic}
\newcommand{\Rd}{\mathbb R^d}

\newtheorem{definition}{Definition}

\newtheorem{proposition}{Proposition}
\newtheorem*{proposition*}{Proposition}

\newtheorem{theorem}{Theorem}
\newtheorem*{theorem*}{Theorem}
\newtheorem{lemma}{Lemma}
\newtheorem{assumption}{Assumption}

\crefname{assumption}{Assumption}{Assumptions}

\newcounter{numrellocal}% Local counter for numering relations
\renewcommand{\thenumrellocal}{\roman{numrellocal}}% Counter numrellocal uses lowercase roman numerals
\newcounter{numrelglobal}% Global counter for numering relations
\makeatletter
\newcommand{\numrel}[2]{% Relation numbering
  \stepcounter{numrellocal}% Increment local counter
  \refstepcounter{numrelglobal}% Increment global counter and create correct reference hook, with label-text from local counter
  \ltx@label{#2}% Label numrel counter
 \overset{(\thenumrellocal)}{#1}% Print counter + relation
}

\title{KALE Flow: A Relaxed KL Gradient Flow for Probabilities with Disjoint Support}

\author{%
  Pierre Glaser \\
Gatsby Computational Neuroscience Unit\\
  University College London\\
  \texttt{pierreglaser@gmail.com} \\
  \And
  Michael Arbel \\
  Universit\'{e} Grenoble Alpes, Inria,  CNRS, \\
 Grenoble INP, LJK,38000 Grenoble, France\thanks{Work mostly completed at the Gatsby Unit.}\\
  \texttt{michael.n.arbel@gmail.com} \\
  \AND
  Arthur Gretton\\
  Gatsby Computational Neuroscience Unit\\
  University College London\\
  \texttt{arthur.gretton@gmail.com}
}

\begin{document}

\vspace*{-8em}

\maketitle

\begin{abstract}
We study the gradient flow for a relaxed approximation to the Kullback-Leibler (KL) divergence
between a moving source and a fixed target distribution.
This approximation, termed the
KALE (KL Approximate Lower bound Estimator), solves a regularized version of
the Fenchel dual problem defining the KL over a restricted class of functions.
When using a Reproducing Kernel Hilbert Space (RKHS) to define the function
class, we show that the KALE continuously interpolates between the KL and the
Maximum Mean Discrepancy (MMD). Like the MMD and other Integral Probability
Metrics, the KALE remains well-defined for mutually singular
distributions. Nonetheless, the KALE inherits from the limiting KL a greater 
sensitivity to mismatch in the support of the distributions, compared with the MMD. These two properties make the
KALE gradient flow particularly well suited when the target distribution is supported on a low-dimensional manifold. Under an assumption of sufficient smoothness of the trajectories, we show the global convergence of the KALE flow. We propose a particle implementation of the flow given initial samples from the source and the target distribution, which we use to empirically confirm the KALE's properties.
\end{abstract}

\section{Introduction}
\label{introduction} 
We consider the problem of transporting probability mass from a source distribution $\mathbb{P}$ to a target distribution $\mathbb{Q}$ using a Wasserstein gradient flow in probability space. 
When the density of the target is well-defined and available, the Wasserstein gradient flow of the Kullback-Leibler (KL) divergence provides a simple way to transport mass towards the target through the Fokker-Planck equation as established in the seminal work of \cite{jordan1998variational}. 
Its time discretization yields a practical algorithm, the  Unadjusted Langevin Algorithm (ULA), which comes with strong convergence guarantees \citep{durmus2018analysis,dalalyan2019user}. A more recent gradient flow approach, Stein Variational Gradient Descent (SVGD) \cite{liu2017stein}, also leverages the analytic expression of the density and constructs a gradient flow of the KL, albeit using a metric different from the Wasserstein metric.

The KL divergence is of particular interest due to its information theoretical interpretation \citep{Shannon:1948} and its use in Bayesian Inference \citep{Blei:2017}. The KL defines a strong notion of convergence between probability distributions, and as such is often widely used for learning generative models, through Maximum Likelihood Estimation \citep{Daniels:1961}. Using the KL as a loss requires  knowledge of the density of the target, however; moreover, this loss is  well-defined only when the distributions share the same support. Consequently, we cannot use the KL in settings where the probability distributions are mutually singular, or when they are only accessible through samples. 
In particular, the Wasserstein gradient flow of the KL in these settings is ill-defined.

Recent works have considered the gradient flow of 
Integral Probability Metrics (IPM) \cite{Mueller97} instead of the KL, in settings where only samples (and not the density) of the target are known. This includes the Maximum Mean Discrepancy (MMD) \citep{arbel_maximum_2019_1} and the Kernelized Sobolev Discrepancy (KSD) \citep{Mroueh:2019,mroueh2020unbalanced}. 
 One motivation for considering these particle flows is their connection with the training of Generative Adversarial Networks (GANs) \cite{gans} using IPMs such as the Wasserstein distance \citep{towards-principled-gans,wgan-gp,sinkhorn-igm}, the MMD \citep{gen-mmd,Li:2015,li_mmd_2017,cramer-gan,Binkowski:2018,arbel2018gradient1} or the Sobolev discrepancy \cite{Mroueh:2017}. 
 As discussed in  \citep[Section 3.3]{Mroueh:2019}, these flows define update equations that are similar to those of a generator in a GAN. Thus, studying the convergence flows can provide helpful insight into conditions for GAN convergence, and ultimately, improvements to GAN training algorithms. 
A second motivation lies in the connection between the training dynamics of infinitely wide 2-layer neural networks and the Wasserstein gradient flow of particular functionals \citep{rotskoff2018neural}.  Thus, analyzing the asymptotic behavior of such flows \citep{mei2018mean,sirignano2018mean,chizat_global_2018} can ultimately provide convergence guarantees for the training dynamics of neural networks. Establishing such results remains challenging for some classes of IPMs, however, such as the MMD  \citep{arbel_maximum_2019_1}.

In this paper, we construct the gradient flow of a relaxed approximation of the KL, termed the KALE (KL Approximate Lower bound Estimator). Unlike the KL, the KALE is well-defined given any source and target, regardless of their relative absolute continuity. The KALE  is obtained by solving a regularized version of the Fenchel dual problem defining the KL, defined over a restricted function class \cite{nguyen_estimating_2010,arbel_generalized_2020_1}, and can be estimated solely from samples from the data. 
   The version of the KALE we consider in this work 
   benefits from two important features that are crucial for defining and analyzing a \emph{relaxed} gradient flow of the KL. (1) We define the function class to be a Reproducing Kernel Hilbert Space (RKHS). This makes the optimization problem defining the KALE convex and allows for practical algorithms computing it. (2) We consider a regularized version of the problem defining the KALE, thus providing a simpler expression for the gradient flow by virtue of the envelope theorem \cite{Milgrom:2002}. In \cref{sec:kale}, we review the KALE, and show that it is a divergence that metrizes the weak convergence of probability measures, while interpolating between the KL and the MMD depending on the amount of regularization. 
 We then construct in \cref{sec:kale_flow} the Wasserstein Gradient Flow of the KALE, and we show global convergence of the KALE flow provided that the trajectories are sufficiently regular.  In \cref{sec:particle_descent}, we introduce the \emph{KALE particle descent} algorithm as well as a practical way to implement it. 
 In \cref{sec:experiments}, we present the results obtained by running the KALE particle descent algorithm on a set of problems with different geometrical properties. We show empirically that the sensitivity to support mismatch of the KALE inherited from the KL leads to well-behaved trajectories compared to the MMD flow, making the KALE flow a desirable alternative when a KL flow cannot be defined.
 \paragraph{Related work.}
The Fenchel dual formulation of the KL, and more generally $ f $-divergences, has a rich history in Machine Learning: \cite{nguyen_estimating_2010} relied on this dual formulation to estimate the KL between two probability distributions when their density ratios belong to an approximating class. They derived a plug-in estimator for the KL which comes with convergence guarantees. 
In the context of GANs, \cite{Nowozin:2016} used the Fenchel dual representation of $f$-divergences, of which the KL is a particular instance, as a GAN critic. Later, \cite{mescheder_adversarial_nodate} used Fenchel duality to estimate the KL in the context of Variational Inference (VI) when the variational distribution is chosen to be an implicit model, thus allowing more flexible models at the expense of tractability of a  KL term appearing in the expression of the ELBO. In both the GAN and VI settings, the function class defining the $f$-divergence was restricted to neural networks.
Recently, \cite{arbel_generalized_2020_1} showed that controlling the smoothness of such a function class results in a divergence, the KL Approximate Lower bound Estimator (KALE), that metrizes the \emph{weak convergence of distributions} \cite{dudley:analysis}, unlike the KL which defines a stronger topology \cite{Van-Erven:2014}. The KALE is therefore well-suited for learning Implicit Generative Models which are only accessible through sampling, as advocated in \cite{arjovsky_wasserstein_2017}. When neural network classes are used, however, the method has no optimization guarantees, as the dual problem becomes non-convex due to the choice of the function class. This is unlike our setting  \citep[and that of][]{nguyen_estimating_2010}, since our dual problem is strongly convex and comes with guarantees. In parallel to work related to $ f $-divergences, \cite{hsieh2018mirrored,bubeck2018sampling,salim2020primal,ahn2020efficient} have investigated the task of sampling in the case where the source and the target have disjoint supports. Again, unlike our setting, these works assume that the log-density of the target distribution is known.

  \section{Interpolating between KL and MMD using KALE}
  \label{sec:kale}
In this section, we introduce the \emph{KALE}, a relaxed approximation of the KL divergence. Although we will use the KALE to define a relaxed KL gradient flow, we show in this section that the KALE is an object of independent interest outside the gradient flow setting: indeed, it is 
a valid \emph{probability divergence} that metrizes the weak
convergence of probability distributions, and interpolates between the KL
and the Maximum Mean Discrepancy. 
\paragraph{Mathematical details and notation} 
We start by introducing some notation.  We denote by $ \mathcal  P(\Rd) $ the set of probability
measures defined on $ \mathbb{R}^d $ endowed with its Borelian $ \sigma $-algebra, and by $ \mathcal  P_2(\Rd) \subset \mathcal  P(\Rd) $ the set of elements of $\mathcal  P(\Rd)$ with finite second moment.
Weak convergence of a sequence of
probability measures $ (\mathbb{ P }_n)_{n \geq 0} $ towards $ \mathbb{ P } $ is written $ \mathbb{
P }_n \rightharpoonup \mathbb{ P} $. 
A positive definite kernel on the set $ \Rd
$ will be denoted $ k: \Rd \times \Rd \longmapsto \mathbb{R} $, with RKHS $ \mathcal H $.
The Dirac delta measure for $ x \in \Rd $ will be written $ \delta_{x} $. We denote by $C^{\infty}_c(\Rd \times (0, +\infty))$ the set of infinitely differentiable functions with compact support on $ \Rd\times (0, +\infty)$,
and  by $C^0_b(\Rd) $ the set of continuous bounded functions from $ \Rd $ to $ \mathbb{R} $. Sets of $ N $ points in $ \Rd $ will be indexed
using a superscript $ \{ x^{(i)} \}_{i=1}^{N}  $, while a sequence of points
in $ \Rd $ will use a subscript: $ \left ( x_n \right )_{n \in
\mathbb{N}} $.  If random, elements of such sets $ X^{(i)} $ or iterates of
such sequences $ X_n $ will be capitalized. If not, they will be kept in
lower-case. For the sake of notational lightness, the choice of the norm used for a specific object (vectors, functions, operators) will be specified with a subscript (e.g. $ \| h \|_{ \mathcal  H} $ for the RKHS norm) only if the said choice is not obvious from the context. This remark also holds when referring to the null element of a vector space ($0_{\mathcal  H} $, $ 0_{\mathbb{R}^d} $, ...).

\subsection{The KL Approximate Lower bound Estimator (KALE)}
\label{sec:KALE}
The central equation to derive the KALE is the (Fenchel) dual formulation of the KL \citep[Lemma 9.4.4]{ambrosio2008gradient}:
\begin{equation}
\label{eq:FL-KL}
\begin{aligned}
  \text{KL}(\mathbb{ P } \mid \mid \mathbb{ Q })  &= \sup_{ h \in C^0_b(\Rd) }
\left \{ 1 + \int_{  }^{  } h \text{d} \mathbb P - \int_{  }^{  } e^{h}\text{d} \mathbb Q \right \}.
\end{aligned}
\end{equation}
KALE is obtained from \cref{eq:FL-KL} by
restricting the variational set to an RKHS $\mathcal{H}$ with reproducing kernel $k$, and by adding a penalty to the objective that controls the RKHS norm of the test function $h$. This regularization ensures that the KALE is well-defined for a broader class of probabilities compared to the KL, even when $\mathbb{P}$ and $\mathbb{Q}$ are mutually singular. Its complete definition is stated below:
\begin{definition}[KALE]
\label{def:kale}
  Let $ \lambda > 0 $, and $ \mathcal  H $ be an RKHS with kernel k. The Kullback-Leibler
  Approximate Lower bound Estimator (KALE) is given by:
  \begin{equation}
    \label{def:KALE}
  \begin{aligned}
    \hspace{-1em}\text{KALE}(\mathbb{ P } \mid \mid \mathbb{ Q }) &= \left ( 1 + \lambda \right )
    \max_{ h \in \mathcal  H  }
\left \{ 1 + \int_{  }^{  } h \text{d} \mathbb P - \int_{  }^{  } e^{h}\text{d} \mathbb Q - \frac{\lambda}{2} \left \|h \right \|^2_{\mathcal  H} \right \}.
  \end{aligned}
  \end{equation}
\end{definition}
The $(1 + \lambda)$ scaling will prevent a degenerate decay to 0 in the large $\lambda$ regime (see \cref{prop:kale-asymptotic-mmd}). The definition we consider here also differs from the one in  \cite{arbel_generalized_2020_1}, which first finds the optimal function $h^{\star}$ solving \cref{def:KALE}, and then defines KALE by evaluating the KL objective in \cref{eq:FL-KL}, thereby discarding the regularization term when evaluating the divergence.

\paragraph{Mathematical Assumptions} 
To prove the theoretical results stated in this work, we will make the following basic assumptions on the kernel $ k $:

\begin{assumption}[Boundedness]
\label{assump:bounded-kernel}
There exists $ K > 0 $ such that $ k(x, x) \leq K $, for all $ x \in \Rd $.
\end{assumption}

\begin{assumption}[Smoothness]
\label{assump:smooth-kernel}
The kernel is $ 2
  $-times differentiable in the sense of \cite[Definition
  4.35]{steinwart2008support}: for all $ i, j \in \left \{ 1, \dots, d \right
  \}  $ $ \partial_i \partial_{i+d} k $ and $ \partial_i \partial_j
  \partial_{i+d} \partial_{j+d} k $ exist. Moreover, we have: $ \left \|
  \nabla_{ 1 } k_x  \right \|^2 \overset{\Delta}{=} \sum_{ i=1 }^{ d } \|
  \partial_i k_x \|^2 \leq K_{1d}$ and $ \| \boldsymbol{H}_{1} k_x  \|^2 = \sum_{
  i, j =1 }^{ d } \| \partial_{i} \partial_{j} k_{x}\|^2  \leq K_{2d} $, where $
  d $ indicates an expected scaling with dimension.
\end{assumption}

\cref{assump:bounded-kernel} guarantees the integrability of the
objects intervening in KALE, and implies boundedness of the RKHS functions.
\cref{assump:smooth-kernel} guarantees first and second order smoothness of the RKHS functions,
a property invoked to control the KALE flow trajectories. Indeed, both the differential and the hessian of any $ f \in \mathcal  H$ can now be bounded in operator norm: using the Cauchy-Schwarz inequality and the kernel reproducing derivative property \cite[Corollary 4.36]{steinwart2008support}, we have: $ \lvert \partial_i f(x) \rvert  \leq \left \| \partial_i k_x \right \| \left \| f \right \| $  and $ \lvert \partial_i \partial_j f(x) \rvert  \leq \left \|\partial_i \partial_j k_x \right \|\left \|f \right \| $, implying $ \| \nabla_{  } f(x) \| \leq \sqrt { K_{1d} }\left \| f \right \|$, and $\left \|
\boldsymbol{H}(f(x))\right \|_{\text{Op}} \leq \left \| \boldsymbol H(f(x)) \right \|_{\text{F}} \leq \sqrt {K_{2d} } \left \| f\right \|$.

\paragraph{KALE is a probability divergence}
We first show that $ \text{KALE} $ is a probability divergence, and presents topological
properties compatible with its use in generative models, such as GANs and Adversarial
VAEs: weak continuity, and metrizing the weak convergence of probability distributions.  We recall that a
functional $ \mathcal  D(\cdot \mid \mid \cdot) $ is a probability divergence if both 
$ \mathcal  D(\mathbb{ P } \mid \mid \mathbb{ Q }) \geq 0$ and
$ \mathcal  D(\mathbb{ P } \mid \mid \mathbb{ Q }) = 0 \iff \mathbb{ P } = \mathbb{ Q }
$, for any $\mathbb P, \mathbb Q \in \mathcal  P(\mathbb{R}^d)$.

\begin{theorem}[Topological properties of $ \textrm{KALE} $]
\label{thm:continuity}
Let $ \mathbb{ P }, \mathbb{ Q } \in \mathcal P(\Rd)$. Let $ \left ( \mathbb {
  P}_n \right )_{n \geq 0}$ be a sequence of probability measures. Then, under \cref{assump:bounded-kernel}:
  \begin{enumerate}[(i)]
    \item KALE is weakly continuous: $
      \mathbb{ P }_n \rightharpoonup \mathbb{ P } \Longrightarrow \lim\limits_{ n  \to \infty }\textrm{KALE}(\mathbb{ P }_n \mid \mid
      \mathbb{ Q })  = \textrm{KALE}(\mathbb{ P } \mid \mid \mathbb{ Q })$
    \item If $ k $ is universal
     \cite{simon-gabriel_kernel_2019}, then for any $ \lambda > 0 $, $ \textrm{KALE} $
     is a probability divergence. Moreover, $ \textrm{KALE} $ metrizes the weak
     topology between probability measures with finite first order moments.
  \end{enumerate}
\end{theorem}

Central to the proof of all points in this theorem is a link between KALE and
the MMD witness function $ f_{\mathbb{ P }, \mathbb{ Q }}$, which we report in
the next lemma.
We recall that given an RKHS $ \mathcal  H $ associated to a kernel $ k $, and
two probability distributions $ \mathbb{ P } $ and $ \mathbb{ Q } $, the MMD
is defined as the RKHS norm of the difference of mean embeddings of $ \mathbb{ P } $ and $ \mathbb{ Q } $:
\begin{equation}
\label{eq:mmd}
\begin{aligned}
  &\text{MMD}(\mathbb{ P } \mid \mid \mathbb{ Q }) = \left \|f_{\mathbb{ P }, \mathbb{ Q
  }} \right \| \quad 
  (f_{\mathbb{ P }, \mathbb{ Q }} = \int_{  }^{  } k(x, \cdot) \text{d} \mathbb P - \int_{  }^{  }k(x, \cdot) \text{d} \mathbb Q \overset{\Delta}{=} \mu_{\mathbb{ P }} - \mu_{\mathbb{ Q }}).
\end{aligned}
\end{equation}

\begin{lemma}
\label{lemma:link-kale-mmd-witness-function}
Let $ \mathbb{ P } $, $ \mathbb{ Q } \in \mathcal  P(\Rd) $, and $ \mathcal  K: \mathcal H  \longmapsto \mathbb{R} $ be the objective maximized by $ \text{KALE} $, e.g. $ \mathcal  K(h) = 1 + \int_{  }^{  } h \text{d} \mathbb P - \int_{  }^{  } e^{h} \text{d} \mathbb Q - \frac{\lambda}{2} \left \|h \right \|^2 $. Then, under \cref{assump:bounded-kernel}, $ \mathcal  K $ is Fr\'{e}chet differentiable.
Moreover, the following relationship holds:
  \begin{equation*}
  \label{eq:link-mmd-witness-function}
  \begin{aligned}
    \nabla_{  } \mathcal  K(0) = f_{\mathbb{ P }, \mathbb{ Q }} 
  \end{aligned}
  \end{equation*}
\end{lemma}
Intuitively, noting that $ \mathcal  K(0) = 0 $, \cref{lemma:link-kale-mmd-witness-function} ensures that $ \text{KALE} $ presents ``equivalent'' regularity and discriminative properties to those of $ \text{MMD} $ (a divergence which is itself, under the assumptions of this theorem, weakly continuous and that metrizes the weak convergence of probability distributions).
The proof of the second point of \cref{thm:continuity} is inspired by
\cite{arbel_generalized_2020_1}, which in turn derives from \cite{zhang2018on,NIPS2017_4491777b},
and is adapted to account for the extra norm penalty
term in the version of the KALE in this paper.

\paragraph{Interpolating between the MMD and the KL using the KALE} 

The KALE includes a positive regularization parameter $\lambda $, inducing two
asymptotic regimes: $ \lambda  \to 0 $ and $ \lambda  \to \infty $. In these
regimes, the KALE asymptotically recovers on the one hand the KL divergence,
and on the other hand the MMD.

\begin{proposition}[Asymptotic properties of $ \textrm{KALE} $]
\label{prop:kale-asymptotic-mmd}
  Let $ \mathbb{ P }, \mathbb{ Q }  \in \mathcal  P(\Rd)$. Then,
  under \cref{assump:bounded-kernel}, the following result holds: 
\begin{equation}
    \begin{aligned}
       \lim_{ \lambda  \to +\infty }\textrm{KALE}(\mathbb{ P } \mid \mid \mathbb{ Q }) 
       = \frac{1}{2} \text{MMD}^2(\mathbb{ P } \mid \mid \mathbb{ Q }).
     \end{aligned}
     \end{equation}
     Suppose additionally that $ \log \frac{ \text{d} \mathbb P }{ \text{d} \mathbb Q } \in \mathcal  H $. Then,
     \begin{equation}
     \label{eq:kale-asymptotic-kl}
     \begin{aligned}
       \lim_{  \lambda  \to 0 } \text{KALE}(\mathbb{ P } \mid \mid \mathbb{ Q
       }) = \text{KL}(\mathbb{ P } \mid \mid \mathbb{ Q }).
     \end{aligned}
     \end{equation}
\end{proposition}

\cref{prop:kale-asymptotic-mmd} shows that the MMD can be seen as solving a
degenerate version of the KL objective.  \cref{eq:kale-asymptotic-kl} is
natural given the original definition of the KALE, and highlights the
continuity of the KALE objective w.r.t the regularization parameter $ \lambda
$.
Both the MMD and the KL  exhibit limitations when used for defining gradient
flows, however: as discussed in
\citep{arbel_maximum_2019_1,feydy_interpolating_2018, bottou2018geometrical}, the
MMD induces a ``flat'' geometry, making its use in generative models tricky
\cite{arbel2018gradient1}. On the other hand, the $ \text{KL} $ comes with
stronger convergence guarantees \cite{ambrosio2008gradient}, but its use in
sampling algorithms is limited to cases where the target distribution has a
density, discarding cases satisfying the widely known \emph{manifold
hypothesis} \citep{NIPS2010_8a1e808b,bottou2018geometrical,Cayton05algorithmsfor}, stating
that typical high dimensional data used in machine learning are distributed on
a lower-dimensional manifold. For this reason, we argue that the true interest
of the KALE does not lie in its interpolation properties, but rather in the
geometry it generates at intermediate values of $ \lambda $.  

\paragraph{The KALE's dual objective} 
Interestingly, the $ \text{KALE} $ itself admits a dual formulation, with a
strong connection to the original KL expression:
\begin{equation}
\label{eq:KALE-primal}
\begin{aligned}
  \text{KALE}(\mathbb{ P } \mid \mid \mathbb{ Q }) &= \min_{ f > 0} \int_{  }^{
  } \left ( f (\log f - 1) + 1 \right )  \text{d} \mathbb Q + \frac{1}{2
  \lambda} \left \| \int_{  }^{  } f(x)k(x, \cdot) \text{d} \mathbb Q(x) - \mu_{\mathbb{ P
}} \right \|_{\mathcal  H}^2 \\
{h}^{\star} &= \int_{  }^{  } {f}^{\star}(x) k(x, \cdot) \text{d} \mathbb Q(x) - \mu_{\mathbb{
  P }}
\end{aligned}
\end{equation}
The solution $ {f}^{\star} $ of \cref{eq:KALE-primal} can be seen as an
entropically-regularized density ratio estimate on the support of $ \mathbb{ Q
} $ (additional details on the KALE dual objective are given in the appendix).
\cref{eq:KALE-primal} also yields an elegant estimation procedure, as discussed
below.

\paragraph{Computing $ \text{KALE}(\mathbb{ P } \mid \mid \mathbb{ Q }) $ in
practice} 

As for other IPMs, computing $ \text{KALE}(\mathbb{ P } \mid \mid \mathbb{ Q })
$ for arbitrary $ \mathbb{ P } $ and $ \mathbb{ Q } $ is intractable, and is therefore 
 approximated using a discretization procedure. A common procedure is to
assume access to  samples $\{ Y^{(i)} \}_{i=1}^N$ and $\{ X^{(i)}\}_{i=1}^N$
from $\mathbb{P}$ and $\mathbb{Q}$ and to solve the empirical equivalent of
\cref{eq:KALE-primal} (e.g. \cref{eq:KALE-primal}, but where $ \mathbb{ P } $
and $ \mathbb{ Q } $ are replaced by their plug-in estimators $ \smash{\widehat{
  \mathbb{ P } }^N = \frac{1}{N}\sum_{i=1}^{N} \delta_{Y^{(i)}}}$ and $ \smash{\widehat{
\mathbb{ Q } }^N = \frac{1}{N}\sum_{i=1}^{N} \delta_{X^{(i)}}}$). This empirical equivalent
is written

\begin{equation}
\label{eq:kale-dual-est}
\begin{aligned}
  \min_{ f > 0 } \quad  \frac{1}{N}\sum_{ i=1 }^{ N } f( X^{(i)}) \log  ( f(X^{(i)})  ) - f(X^{(i)}) + 1
                       + \frac{1}{2 \lambda} \Big \| \frac{1}{N}\sum_{ i=1 }^{ N
			} f(X^{(i)}) k(X^{(i)}, \cdot) - \mu_{ \widehat{  \mathbb{ P } }} \Big \|^2_{\mathcal  H}
\end{aligned}
\end{equation}
which is a strongly convex $ N $-dimensional problem, and can be solved using
standard euclidean optimization methods. By adapting arguments of
\cite{arbel_generalized_2020_1}, it can be shown that the discrepancy between the
KALE's empirical and population value, $ \lvert \text{KALE}( \widehat{ \mathbb{
P } }^N \mid \mid \widehat{ \mathbb{ Q } }^N) - \text{KALE}(\mathbb{ P } \mid
\mid \mathbb{ Q }) \rvert  $ (often called ``sample complexity''), is at most
$ O(\frac{1}{\sqrt {N}}) $. This rate is identical that of Sinkhorn
divergences \cite{genevay_sample_2019}, another family of
entropically-regularized divergences.

\section{KALE Gradient Flow}
\label{sec:kale_flow}
Having introduced KALE as a relaxed approximation of the KL, we now construct 
 the KALE gradient flow, and assert its well-posedness. We provide
conditions for global convergence of the flow, and discuss its relationship
with the MMD flow and the KL flow. All proofs are given in the appendix.

\subsection{Wasserstein Gradient Flow of the KALE}
\emph{Wasserstein Gradient Flows} of divergence functionals $ \mathcal  F(\mathbb{ P } \mid \mid \mathbb{ Q }) $
aim at transporting mass from an initial probability distribution
$\mathbb{P}_0$ towards a target distribution $\mathbb{Q}$ by following a path
$\mathbb{P}_t$ in probability space. The path is required to dissipate
energy, meaning that $t\mapsto \mathcal  F(\mathbb{P}_t \mid \mid \mathbb{ Q })$
is a decreasing function of time. Additionally, it is constrained to satisfy a
continuity equation that allows only local movements of mass without jumping
from a location to another. This equation involves a time dependent vector
field $V_t$ which serves as a force that drives the movement of mass at any
time $t$:
\begin{equation}
\label{eq:wfg}
\begin{aligned}
  \partial_{t} \mathbb{ P }_t  + \textrm{div}(\mathbb{P}_t V_t) = 0.
\end{aligned}
\end{equation}
\cref{eq:wfg} holds in the \emph{sense of distributions}, meaning that for any test
function $ \varphi \in  C^{\infty}_c(\Rd \times (0, + \infty))  $, we
have:
\begin{equation*}
\begin{aligned}
    \int_{  }^{  }\partial_t \varphi(x, t)\text{d}\mathbb{ P }_t\text{d}t  + \int_{  }^{
    }\left \langle \nabla_{ x  }  \varphi(x, t),  V_t\right
    \rangle_{\mathbb{R}^d} \text{d}
    \mathbb P_t \text{d}t = 0.
\end{aligned}
\end{equation*}
The Wasserstein gradient flow of a functional $ \mathcal  F $ is then obtained
by choosing $V_t$ as the gradient of \emph{first variation} of $ \mathcal  F $, defined as the G\^{a}teaux derivative of $ \mathbb{ P } $ along the direction $ \mathcal  \chi $,
\begin{equation*}
\begin{aligned}
	\mathcal D_{\mathbb{ P }}\mathcal  F(\mathbb{ P }; \chi)
	\overset{\Delta}{ = } \lim_{\epsilon \rightarrow 0}
	\epsilon^{-1}\left(\mathcal{F}(\mathbb{P}+\epsilon \chi)
	-\mathcal{F}(\mathbb{P}) \right),
\end{aligned}
\end{equation*}
where $ \int_{  }^{  } d \chi = 0 $, and provided that such a limit exists.
This choice recovers a particle
\emph{Euclidean gradient flow} when $\mathbb{P}_0$ is a finite sum of Dirac
distributions, and can thus be seen as a natural extension of gradient flows to
the space of probability distributions
\citep{ambrosio2008gradient,Villani:2004,Villani:2009}. In the next
proposition, we show that the functional $ \mathbb{ P }  \longmapsto \text{KALE}(\mathbb{ P }
\mid \mid \mathbb{ Q }) $ admits a well-defined gradient flow. 
\begin{proposition}[$ \text{KALE} $ Gradient Flow]
\label{prop:KALE-gf}
  Let $ \lambda > 0 $, and $ \mathbb{ P }_0, \mathbb{ Q } \in \mathcal  P_2(\Rd)$.
  Under Assumptions \ref{assump:bounded-kernel} and \ref{assump:smooth-kernel}, 
  the Cauchy problem
  \begin{equation}
  \label{eq:KALE-gf}
  \begin{aligned}
      \partial_t \mathbb{ P }_t - \text{div}  ( \mathbb{ P }_t (1 + \lambda)\nabla_{  } {h}_t^{\star})  = 0, \quad
    \mathbb{ P }_{t=0} = \mathbb{ P }_0 ,
    \end{aligned}
  \end{equation}
  where $h^{\star}_t$ is the unique solution of
  \begin{align}
  \label{eq:first-variation-continuous-time}
  	{h}^{\star}_t = \arg \max_{ h\in \mathcal{H} } \left \{ 1 + \int_{  }^{  } h \text{d} \mathbb P_t -
    \int_{  }^{  } e^{h} \text{d} \mathbb Q - \frac{ \lambda }{ 2 } \left \|h \right
  \|^2 \right \},
  \end{align}
  admits a \emph{unique} solution $ (\mathbb{ P }_t)_{t \geq  0} $, which is
  the \emph{Wasserstein Gradient Flow} of the $ \text{KALE} $.
\end{proposition}

\subsection{Convergence properties of the KALE flow}
\cref{prop:KALE-gf} hints at a connection between the $ \text{KALE}$ flow
and the $ \text{MMD} $ flow, which solves:
\begin{equation}
\label{eq:mmd-flow}
\begin{aligned}
\partial_t \mathbb{ P }_t - \text{div} \left ( \mathbb{ P }_t \nabla_{  }
f_{\mathbb{ P }_t, \mathbb{ Q }} \right ) = 0, \quad  \mathbb{ P }_{t=0} = \mathbb{ P }_0
\end{aligned}
\end{equation}
 The $ \text{MMD} $ flow and the $ \text{KALE} $
flow thus differ in the choice of witness function characterizing their
velocity field. A convergence analysis of the MMD flow was proposed for a wide
range of kernels in \cite{arbel_maximum_2019_1} using inequalities of Lojasiewicz
type; in particular, the MMD flow is guaranteed to converge provided that the
quantity $ \mathbb{ P }_t - \mathbb{ Q } $ remains bounded in the
\emph{negative Sobolev distance}  $ \smash{\| \mathbb{ P }_t - \mathbb{ Q } 
\|_{\dot{H}^{-1}( \mathbb{ P }_t)}} $ \cite{otto_generalization_2000}.
We recall that the  negative \emph{weighted negative Sobolev distance} \cite{arbel_maximum_2019_1} between $ \mu $ and $ \nu $ is defined as:
\begin{equation*}
\begin{aligned}
  \left \| \mu - \nu \right \|_{ \dot{H}^{-1}(\mathbb{ P })} &= \sup_{ \left \|f
  \right \|_{\dot{H}(\mathbb{ P })} \leq 1 } \Big \lvert \int_{  }^{  } f
\text{d}(\mu - \nu) \Big \rvert,
\end{aligned}
\end{equation*}
which is obtained by duality with the weighted Sobolev
semi-norm $\smash{\| f \|_{\dot{H}(\mathbb{ P })} = ( \int_{  }^{
} \| \nabla_{  } f  \|^2 \text{d} \mathbb P )^{\frac{1}{2}}}$. Note the important role of the latter quantity in the energy
dissipation formula of the KALE gradient flow:
\begin{equation}
\begin{aligned}
  \frac{ \text{d} \text{KALE}(\mathbb{ P }_t \mid \mid \mathbb{ Q }) }{ \text{d}t } = -\int_{  }^{  } (1 + \lambda)^2\left \| \nabla_{  }  {h}^{\star}  \right \|^2 \text{d} \mathbb P = - (1 + \lambda)^2\left \| {h}^{\star} \right \|^2_{\dot{H}(\mathbb P)}.
\end{aligned}
\end{equation}

In the
next proposition, we extend the condition ensuring the global convergence of
the MMD flow \cite{arbel_maximum_2019_1} to the KALE flow:
\begin{proposition}
\label{prop:KALE-flow-cvg}
  Under Assumptions \ref{assump:bounded-kernel} and \ref{assump:smooth-kernel}, if $
  \left \| \mathbb{ P }_t - \mathbb{ Q } \right \|_{\dot{H}^{-1}(\mathbb{ P
  }_t)} \leq C $ for some $ C > 0 $, then:
   \begin{equation*} 
  \begin{aligned}
    \text{KALE}(\mathbb{ P }_t \mid \mid \mathbb{ Q }) \leq \frac{C}{C \text{KALE}(\mathbb{ P }_0 \mid \mid \mathbb{ Q }) + t}.
  \end{aligned}
  \end{equation*}
\end{proposition}

\cref{prop:KALE-flow-cvg} ensures a convergence rate in
$\mathcal O(1/ t)$ provided that 
$\| \mathbb{ P }_t - \mathbb{ Q } \|_{\dot{H}^{-1}(\mathbb{ P
}_t)}$ remains bounded. This convergence rate is slower than the linear rate of
the KL along its gradient flow \cite{ma_is_2019} and could be an effect of RKHS smoothing.

\section{KALE Particle Descent}
\label{sec:particle_descent}

We now derive a practical algorithm that computes the solution of
a KALE gradient flow, given an initial source-target pair $ \mathbb{ P }_0 $ and
$ \mathbb{ Q }$. Because of the continuous-time dynamics, and the possibly
continuous nature of $ \mathbb{ P }_0 $ and $ \mathbb{ Q } $, solutions of
\cref{eq:KALE-gf} are intractable to compute and manipulate. To address this
issue, we first introduce the \emph{KALE Particle Descent Algorithm} that returns a
sequence $ \smash{(\widehat{ \mathbb{ P } }^{N}_{n})_{n \geq  0}}$ of discrete
probability measures able to approximate the forward Euler discretization of $ \mathbb{ P }_t $ with arbitrary
precision.  Additionally, we show that the KALE particle descent algorithm can
be regularized using \emph{noise injection} \cite{arbel_maximum_2019_1}, which
guarantees  global convergence of the flow under a suitable noise schedule.
All proofs are given in the appendix.

\subsection{The KALE Particle Descent Algorithm}

\paragraph{Time-discretized KALE Gradient Flow}
As a first step towards deriving the KALE particle descent algorithm, let us first consider a
time-discretized version of the KALE gradient flow (\cref{eq:KALE-gf}
and \cref{eq:first-variation-continuous-time}), obtained by applying a
forward-Euler scheme to \cref{eq:KALE-gf} with step size $ \gamma $. This
time-discretized equation is given by
\begin{equation} 
\label{eq:time-discrete-gradient-flow}
\begin{aligned}
  \mathbb{ P }_{n+1} = (I - \gamma (1 + \lambda)\nabla_{  } {h}^{\star}_n )_{\#} \mathbb{ P }_n, \,\, \mathbb{ P }_{n=0} = \mathbb{ P }_0.
\end{aligned}
\end{equation}
The function $h^{\star}_n$ is  a discrete time analogue of
\cref{eq:first-variation-continuous-time}, in that it is  solution to the following optimization problem:
\begin{equation}
\label{eq:first-variation-discrete-time}
  \begin{aligned}
  	{h}^{\star}_{n} = \arg \max_{ h\in \mathcal{H} } \left \{ 1 + \int_{  }^{  } h \text{d} \mathbb P_{n} -
    \int_{  }^{  } e^{h} \text{d} \mathbb Q - \frac{ \lambda }{ 2 } \left \|h \right
  \|^2 \right \}.
\end{aligned}
\end{equation}
The solution $ \mathbb{ P }_n $ of \cref{eq:time-discrete-gradient-flow} is a
sensible approximation of $ \mathbb{ P }_t $: indeed, it can be shown under
suitable smoothness assumptions \cite{santambrogio_optimal_nodate,arbel_maximum_2019_1} that the
piecewise-constant trajectory $ (t  \longmapsto \mathbb{ P }_n \,\,\,\text{if } t \in \left
\lbrack n \gamma, (n+1) \gamma \right ))  $ obtained from the
time-discretized gradient flow of a functional $ \mathcal  F $ will recover the
true gradient flow solution $ \mathbb{ P }_t $ of $ \mathcal  F $ as $ \gamma
\to 0 $.

\paragraph{Approximation using finitely many samples: the KALE particle descent algorithm}

Despite its discrete-time nature, the sequence $ (\mathbb{ P }_n)_{n \geq  0} $ may 
still be intractable to compute: for generic $ \mathbb{ P }_0 $ and $ \mathbb{
Q } $, \cref{eq:first-variation-discrete-time} will contain intractable
expectations and have an infinite dimensional search space. To address this
issue, we propose the \emph{KALE particle descent algorithm}: this algorithm
\emph{approximates} the true time-discrete iterates $ \mathbb{ P }_n$ given $ N $ samples $
\smash{\{ X^{(i)} \}_{i=1}^{N}} $ and $ \smash{\{ Y^{(i)}_0 \}_{i=1}^{N}}  $ of $ \mathbb{ Q }
$ and $ \mathbb{ P }_{0}$, by computing the probabilities $ \smash{\widehat{ \mathbb{ P }
}^{N}_n}$ solving the time-discrete KALE gradient flow arising from the
\emph{empirical} source-target pair $
\widehat{ \mathbb{ Q } }^{N}=\frac{1}{N} \sum_{ i=1 }^{ N } X^{(i)} $ and $
\widehat{ \mathbb{ P } }^{N}_0=\frac{1}{N}\sum_{ i=1 }^{ N } Y^{(i)}_0 $. 
As opposed to $ \mathbb{ P }_n $, it is possible to \emph{exactly} compute $
\widehat{ \mathbb{ P } }^{N}_n $:
indeed, the recursion equation
\cref{eq:time-discrete-gradient-flow} implies that $ \widehat{ \mathbb{ P }
}^{N}_n $ remains discrete for all $ n $. More precisely, we have $ \widehat{
\mathbb{ P } }^{N}_n = \frac{1}{N} \sum_{ i=1 }^{ N }Y^{(i)}_n $, where
\begin{equation}
\label{eq:kale-descent-update-particles}
\begin{aligned}
  Y^{(i)}_{n+1} = Y^{(i)}_{n} - \gamma (1 + \lambda)\nabla_{  } \widehat{ h }^{\star}_n(Y^{(i)}_n),
\end{aligned}
\end{equation}
and $ \widehat{ h }^\star $ is defined as
\begin{equation}
\label{eq:kale-descent-compute-velocity-field}
\begin{aligned}
\widehat{ h }^{\star}_n = \arg \max_{ h \in \mathcal  H} \left \{ \int_{  }^{  } h d \widehat{
\mathbb{ P } }^{N}_n - \int_{  }^{  } h d \widehat{ \mathbb{ Q } }^{N} -
\frac{\lambda}{2} \left \| h \right \|^2_{\mathcal  H}\right \} .
\end{aligned}
\end{equation}
As in the sample-based setting of \cref{eq:kale-dual-est}, $ \widehat{ \mathbb{
P } }^{N}_n $ and $ \widehat{ \mathbb{ Q } }^{N} $ are discrete, meaning that
\cref{eq:kale-descent-compute-velocity-field} reduces to an $ N $- dimensional
convex problem, and $ \widehat{ h }^{\star}_n $ can be tractably computed. The
alternate execution of \cref{eq:kale-descent-update-particles} and
\cref{eq:kale-descent-compute-velocity-field} for a finite number of time steps
defines the \emph{KALE Particle Descent Algorithm}, that we
lay out in Algorithm \ref{alg:particle-descent}.

\paragraph{Consistency of the KALE Particle Descent Algorithm} 
Note that the source of error in the KALE particle descent algorithm lies in the
use of an approximate witness function $ \widehat{ h }^{\star}_n$ instead of
the true, but intractable, $ {h}^{\star}_n $. Indeed, one can show, using the
theory of McKean-Vlasov representative processes \cite{McKean1907}, that the $
n $-th iterates of the sequence defined by:
\begin{equation}
\label{eq:exact-kale-descent}
\begin{aligned}
  \bar{ Y }^{(i)}_{n+1} = \bar{ Y }^{(i)}_n - \gamma (1 + \lambda)\nabla_{  }  {h}_n^{\star} (
  \bar{ Y }^{(i)}_n) ,\,\,\bar{ Y }^{(i)}_0 \sim \mathbb{ P }_0, \,\, 1 \leq i \leq N
\end{aligned}
\end{equation}
are distributed according to the $ n^{\text{th}} $ iterate $ \mathbb{ P }_n $ of the true discrete-time KALE gradient flow solution 
defined in \cref{eq:time-discrete-gradient-flow}. As such, the discrete
probability $ \smash{\bar{ \mathbb{ P } }^N_n = \frac{1}{N}\sum_{ i=1 }^{ N }
\delta_{ \bar{ Y }^{(i)}_n}}$ may be considered as an unbiased space-discretization of
\cref{eq:time-discrete-gradient-flow}.  In the next proposition, we show that the
iterates $ \widehat{ \mathbb{ P }}^N_n $ returned by the KALE particle descent
algorithm can approximate the unbiased $ \bar{ \mathbb{ P } }^N_n $ with
arbitrarily low error.
\begin{proposition}[Consistency of the KALE particle descent]
\label{prop:kale-descent-vs-discrete-kale-flow}
  Let $ \{Y^{(i)}_0\}_{i=1}^{N} \sim \mathbb{ P }_0$. Let $ (\bar{ \mathbb{ P }
    }_n^N)_{n \geq  0} $ be the sequence of discrete probabilities arising from
    \cref{eq:exact-kale-descent} with initial conditions $
    \{Y^{(i)}_{0}\}_{i=1}^{N} $, and let $ (\widehat{ \mathbb{ P } }_n^N)_{n \geq
    0} $ be the sequence arising from \cref{eq:kale-descent-update-particles} with
    the \emph{same} initial conditions $ \{Y^{(i)}_{0}\}_{i=1}^{N} $. Let $n_{\text{max}} \geq 0$. Then, under Assumptions \ref{assump:bounded-kernel} and \ref{assump:smooth-kernel}, for
    all $ n \leq  n_{\max_{  }} $, the following bound holds:
  \begin{equation*}
  \begin{aligned}
    \mathbb{E} W_2( \widehat{ \mathbb{ P } }^{N}_n, \bar{ \mathbb{ P } }^{N}_{n} )
    \leq \frac{ A }{ B\sqrt {N}  }(e^{\gamma B n_{\max_{  }}} - 1)
  \end{aligned}
  \end{equation*}
  with $A = \sqrt {{ 2 K K_{1d}(1 + e^{\frac{8K}{\lambda}}) }} \times
\frac{1}{4 \sqrt {KK_{1d}}  + K_{2d} }$
$ B = \frac{ (1 + \lambda )(4 \sqrt {K K_{1d}} + \sqrt {K_{2d}})  }{ \lambda }$, and $ K, K_{1d}, K_{2d} $ are the constants defined in Assumptions \ref{assump:bounded-kernel} and \ref{assump:smooth-kernel}.
\end{proposition}
\cref{prop:kale-descent-vs-discrete-kale-flow} shows that given a finite time
horizon $ n_{\max_{  }} $, and given sufficiently many samples of $ \mathbb{ P
}_0 $ and $ \mathbb{ Q } $, one can approximate an exact discrete KALE flow
between $ n=0 $ and  $ n=n_{\max_{ }} $ with arbitrary precision.  The proof of
\cref{prop:kale-descent-vs-discrete-kale-flow} (given in
\cref{proof:kale-descent-vs-kale-flow}) relies on the regularity of the KALE
witness function $ x  \longmapsto \widehat{ h }^\star_n(x) $, but also on the
regularity of the mapping $ \widehat{ \mathbb{ P } }_n^{N}  \longmapsto
\widehat{ h }^\star_n $ (using the 2-Wasserstein distance as the metric on $
\mathcal P_2(\mathbb{R}^d) $.

	\begin{algorithm}[H]%\captionsetup{labelfont={sc,bf}, labelsep=newline}\small
	\caption{KALE Particle Descent Algorithm}
	\label{alg:particle-descent}
	\begin{algorithmic}
	     \STATE {\bfseries Input:}  $ \{ Y_0^{(i)}\}_{i=1}^N
	     \sim \mathbb{ P }_0$, $\{ X^{(i)}\}_{i=1}^N \sim \mathbb{ Q } $, \texttt{max\_iter}, $ \lambda $, $ k $,  $ \gamma $
	     \STATE {\bfseries Output}  $ \{ Y^{(i)}_{\texttt{max\_iter}}\}_{i=1}^N$
	\FOR{$n=0$ {\bfseries to} \texttt{max\_iter}$-1$}
	\STATE \texttt{f\_star} $\leftarrow$ \texttt{dual\_solve}$(X^{(1)}, Y_i^{(1)},\dots, X^{(N)}, Y_i^{(N)}, k, \lambda) \quad \quad  \quad \quad \quad \quad \quad \quad  $ \# \texttt{See Eq.6}
      \STATE \texttt{h\_star} $\leftarrow$ \texttt{compute\_log\_ratio}$($\texttt{f\_star}$, X^{(1)}, Y_i^{(1)},\dots, X^{(N)}, Y_i^{(N)}, k, \lambda) \,\,\,\,$   \# \texttt{Ditto}
	   \FOR{$j=1$ {\bfseries to} $N$}
	   \STATE \texttt{v} $\leftarrow$ $ (1 + \lambda)$\texttt{grad}$($\texttt{h\_star}$(Y_i^{(j)}))$
	   \STATE $ Y_{i+1}^{(j)} \leftarrow Y_i^{(j)} - \gamma \times $\texttt{v}$ \quad \quad \quad \quad \quad \quad \quad \quad \quad \quad \quad \quad \quad \quad \quad \quad \quad \quad\quad \quad \quad \quad \,\,\,  $ 
	   \ENDFOR
	   \ENDFOR
	\end{algorithmic}
	\end{algorithm}

\subsection{Regularization of  KALE particle descent using Noise Injection}

In practice, guaranteeing the convergence of the KALE gradient flow (and its
corresponding KALE particle descent) by relying on
the condition given in  \cref{prop:KALE-flow-cvg} is cumbersome for two
reasons: first, this condition is hard to check, and second, it does not tell
us what to do when the condition is not met. Noise injection
\citep{arbel_maximum_2019_1, birdal_synchronizing_2020}
is a practical
regularization technique originally introduced for the MMD flow, that trades
off some of the ``steepest descent'' property of gradient flow trajectories
with some additional smoothness (in negative Sobolev norm) in order to improve
convergence to the target trajectory. We recall that the solution of a (discrete time) noise
injected gradient flow with velocity field $ (1 + \lambda)\nabla_{  } {h}^{\star}_n $ and noise schedule $
\beta_n$ is defined as the sequence $ (\mathbb{ P }_n)_{n \geq  0} $ whose iterates verify:
\begin{equation}
\label{eq:noise-injection-population-equation}
\begin{aligned}
  \mathbb{ P }_{n+1} = \left ( (x, u)   \longmapsto x - \gamma  (1 + \lambda)\nabla_{  }
  {h}^{\star}_n (x + \beta_n u) \right )_{\#} (\mathbb{ P }_n \otimes g),
\end{aligned}
\end{equation}
where $ g $ is a standard unit Gaussian distribution.  As we show in the next
proposition, under a suitable noise schedule, noise injection can also be
applied to ensure global convergence of the KALE flow.

\begin{proposition}[Global Convergence under noise injection dynamics]
\label{prop:KALE-noise-injection}
  Let $ \mathbb{ P }_n $ be defined as
  \cref{eq:noise-injection-population-equation}. Let $ (\beta_n)_{n \geq  0} $
  be a sequence of noise levels, and define  $ \mathcal  D_{\beta_n, \mathbb{ P }_n} =
  \mathbb{E}_{ y \sim \mathbb{ P }_n, u \sim g } \left \| \nabla_{  }
  {h}_n^{\star}(x + \beta_n u) \right \|^2 $ with $ g $ the density of a standard
  Gaussian distribution. Then, under
  Assumptions \ref{assump:bounded-kernel}  and \ref{assump:smooth-kernel}, and for a
  choice of  $ \beta_n $ such that:
  \begin{equation*}
  \begin{aligned}
    \frac{ 8 K_{2d} \beta_n^2 }{ \lambda^2 }  \text{KALE}(\mathbb{ P }_n \mid \mid \mathbb{ Q }) \leq \mathcal  D_{\beta_n, \mathbb{ P }_n}(\mathbb{ P }_n),
  \end{aligned}
  \end{equation*}
  the following holds: $  \text{KALE}(\mathbb{ P }_{n+1} \mid \mid
  \mathbb{ Q }) - \text{KALE}(\mathbb{ P }_n \mid \mid \mathbb{ Q }) \leq
  -\frac{\gamma}{2} (1 - 3 \gamma \sqrt {K K_{2d}} )D_{\beta_n, \mathbb{ P }_n}(\mathbb{ P }_n)$.
  Moreover, if $ \sum\limits_{ i=1 }^{ \infty } \beta_i =
  +\infty $, then $ \lim_{  n  \to \infty  }\text{KALE}(\mathbb{ P }_n \mid
  \mid \mathbb{ Q }) = 0 .$
\end{proposition}
As in \cite{arbel_maximum_2019_1}, convergence of the regularized KALE flow is
guaranteed when the noise schedule satisfies an inequality for all $ n $, which
is hard to check in practice. Nonetheless, we empirically observe that in all
our problems a small, constant noise schedule can help the KALE flow reach a
lower KALE value at convergence.

Let us stress that the noise injection scheme given in \cref{prop:KALE-noise-injection} is a
\emph{population} scheme that includes an intractable convolution.  To use
noise injection in the KALE particle descent algorithm, we approximate this convolution
using a single sample $ U_n^{(i)} $ for each particle update. \cref{eq:kale-descent-update-particles} becomes:
\begin{equation}
\label{eq:kale-descent-update-particles-noise-injection}
\begin{aligned}
  Y^{(i)}_{n+1} = Y^{(i)}_{n} - \gamma (1 + \lambda)\nabla_{  } \widehat{ h }^{\star}_n(Y^{(i)}_n + \beta_n U^{(i)}_n) , \,\,\, U^{(i)}_n \sim \mathcal  N(0, 1).
\end{aligned}
\end{equation}

\paragraph{Implementation}
The particle descent algorithm can be implemented using automatic
differentiation software such as the \verb|pytorch| library in python. This
allows us to easily compute the gradient of the log-density ratio estimate
$\smash{\widehat{h}^{\star}_{n}}$ appearing in the particle update rule
\cref{eq:kale-descent-update-particles}.

Computing $\smash{\widehat{h}^{\star}_{n}}$ can be achieved using 
methods such as gradient descent,  coordinate  descent or higher order
optimization methods such as Newton's method and L-BFGS \cite{l_bfgs_paper}.

\section{Experiments}
\label{sec:experiments}
In this section, we empirically study the behavior of the KALE particle descent algorithm in three settings reflecting different topological properties for the source-target pair:
a pair with a target supported on a hypersurface (zero
volume support), a pair with disjoint supports of positive volume, and a pair
of distributions with a positive density supported on $ \mathbb R^d $.

{\bf KALE flow for targets defined on hypersurfaces} 
Our first example consists in a target \emph{supported} (and uniformly
distributed) on a lower-dimensional surface  that defines three non-overlapping rings. 
The initial source is a Gaussian distribution with a mean in the vicinity of
the target $ \mathbb{ Q } $.  This setting is a perfect candidate to illustrate
the failure modes of both the KL and the MMD when used in particle descent
algorithms: on the one hand, the measures $ \mathbb{ P }_0 $ and $ \mathbb{ Q }
$ are mutually singular, and thus the KL gradient flow from $ \mathbb{ P
}_{0} $ to $ \mathbb{ Q } $ does not exist.
By contrast, the KALE is well-defined in this case, and inherits from the KL an
increased sensitivity to support discrepancy. For that reason, 
we hypothesize that the trajectories of the KALE flow will converge towards a
better limit compared to its MMD flow counterpart.  We sample $ N=300 $ points
from the target and the initial source distribution and run an implementation
of \cref{alg:particle-descent}  for $ n=50000 $ iterations. The complete set of
parameters is given in the appendix.
\begin{figure}[h]
    \includegraphics[width=1\textwidth]{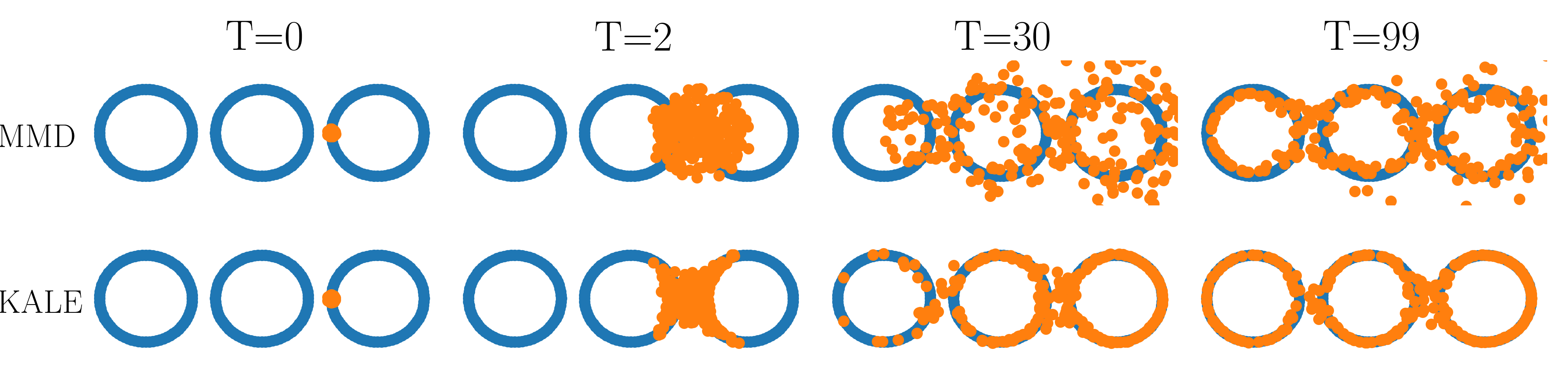}
    {\caption{MMD and KALE flow trajectories for ``three rings'' target}\label{fig:kale_flow_ring}}
\end{figure}
Results are plotted in \cref{fig:kale_flow_ring}. We indeed notice that the
KALE flow trajectory remains close to the target support and recovers the
target almost perfectly. This illustrates the ability of the KALE flow to
\emph{relax} the hard support-sharing constraints of the KL flow into soft
support closeness constraints.  These soft constraints are not present in the
MMD flow, where particles of the source can remain scattered around the plane.

{\bf KALE flow between probabilities with disjoint support} 
In our second example, we consider a source/target pair that are
supported on disjoint subsets each with a finite, positive volume (unlike the
previous example). The support of the source and the target consist
respectively of a heart and a spiral, and the two distributions have a  uniform
density on their support.  Again, because the supports of the source and the
target are disjoint, the KL flow cannot be defined, nor simulated for this pair.
We run a KALE particle descent algorithm, and compare it as before with an MMD
flow, as well as with a ``Sinkhorn descent algorithm''
\cite{feydy_interpolating_2018}. Results are in
\cref{fig:kale_flow_shape_transfer}.

\begin{figure}[h]
    \centering
    \includegraphics[width=1.\textwidth]{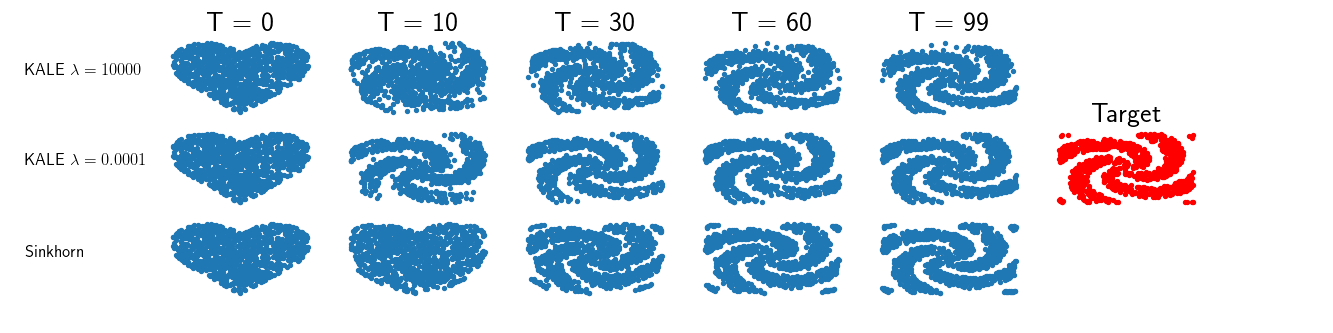}
    \caption{Shape Transfer using the KALE flow}
    \label{fig:kale_flow_shape_transfer}
\end{figure} 

As we can see, the soft support-sharing constraint informing the KALE
flow forces the source to quickly recover the spiral shape, much before
the Sinkhorn and MMD flow trajectories.  However, compared to Sinkhorn, the two
KALE-generated spirals have a harder time recovering outliers, disconnected
from the main support of the spiral.

{\bf KALE flow for probabilities with densities}\label{subsec:experiments-mog}
We consider the setting where the target admits a positive density on $
\mathbb{R}^d $. Hence, unlike in the two
previous examples, the KL gradient flow is well-defined, and can be simulated using the Unadjusted Langevin Algorithm (ULA). 
Echoing the interpolation property of the KALE between the MMD and the KL shown in \cref{prop:kale-asymptotic-mmd}, we propose to investigate whether this property is preserved in a \emph{gradient flow} setting.
We consider a balanced mixture of 4 Gaussians with means located on the 4 corners of the unit square for the target  and a source distribution given by a unit Gaussian in the vicinity of the unit square. We then run KL, MMD, and  KALE flows with different values of $\lambda$, and compute the Wasserstein distance between reference particles at iteration $n$ from either the MMD or KL flow and particles obtained from the KALE flow at the same iteration $n$. The choice of the Wasserstein distance is natural for Wasserstein Gradient Flows.
As shown in \cref{fig:KALE-flow-mog-kl}, 
for ``small'' values of $ \lambda $, particles from a KALE flow remain close to the ULA particles, while for ``large'' ones they remain close to the MMD particles (\cref{fig:KALE-flow-mog-mmd}).
\begin{figure*}[h]
\centering

\begin{subfigure}[t]{0.32\textwidth}
    \centering
    \includegraphics[width=0.99\textwidth]{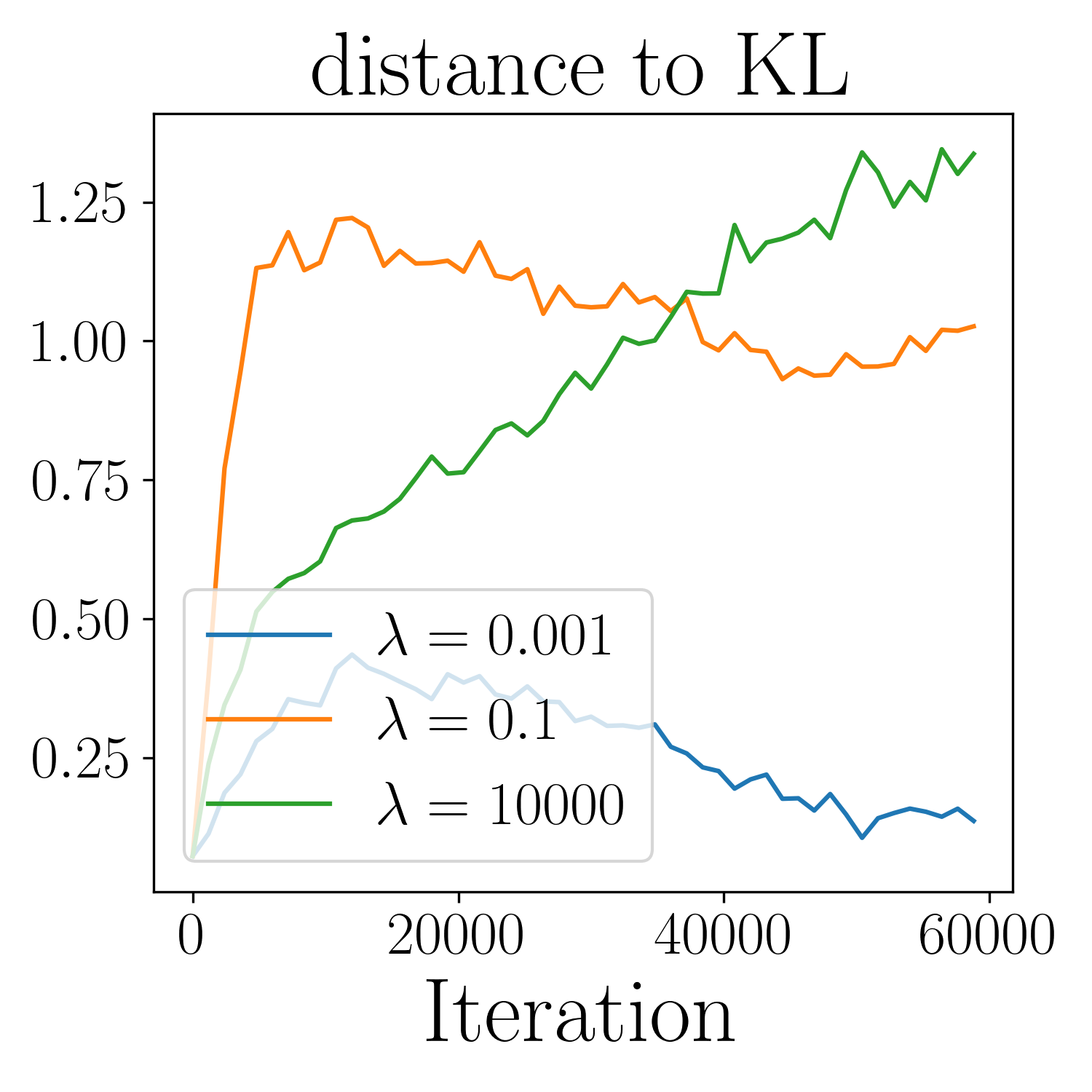} %
    \caption{}\label{fig:KALE-flow-mog-kl}
\end{subfigure} % 
\begin{subfigure}[t]{0.32\textwidth}
    \centering
    \includegraphics[width=0.99\textwidth]{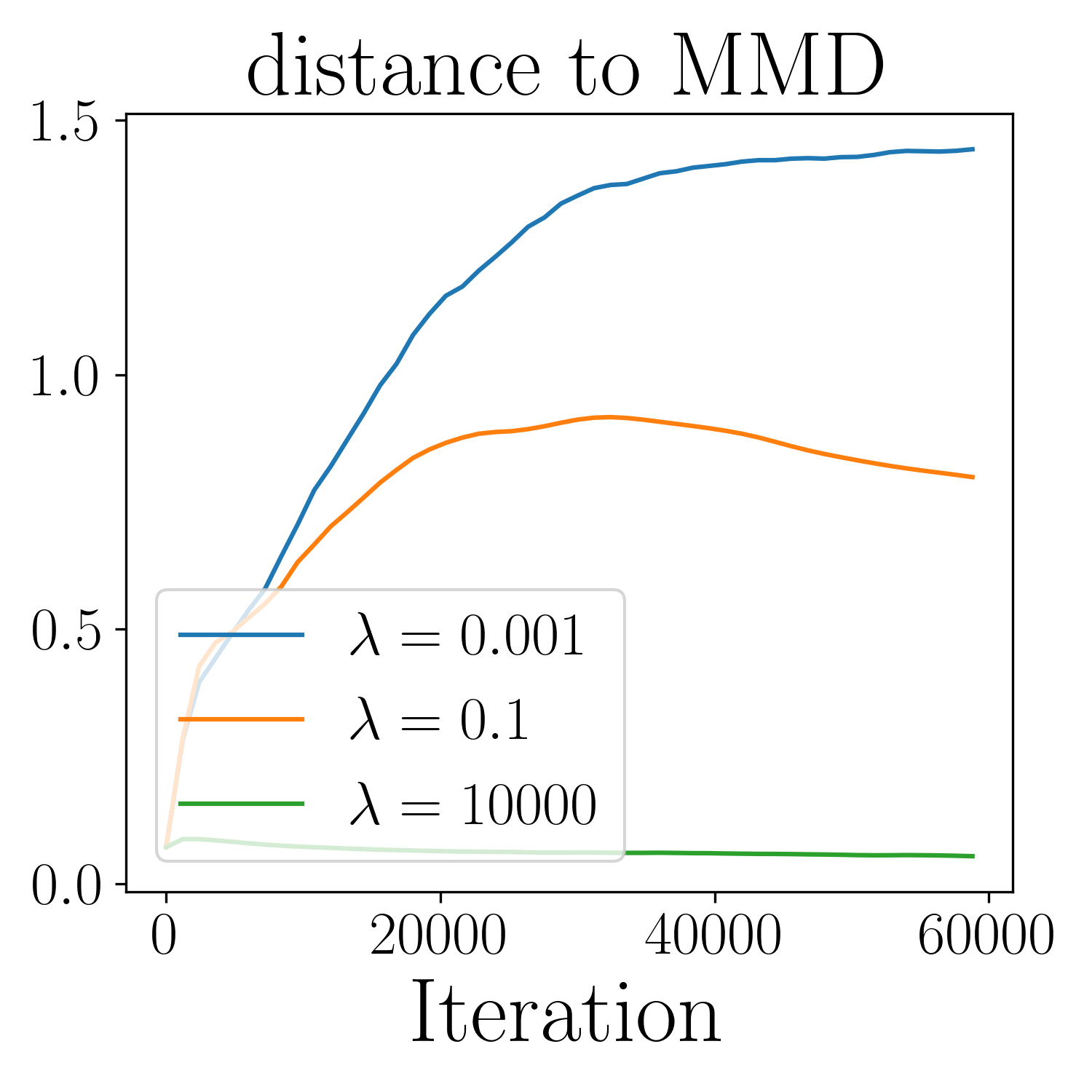} %
    \caption{}\label{fig:KALE-flow-mog-mmd}
\end{subfigure} % 
\begin{subfigure}[t]{0.32\textwidth}
    \centering
    \includegraphics[width=0.99\textwidth]{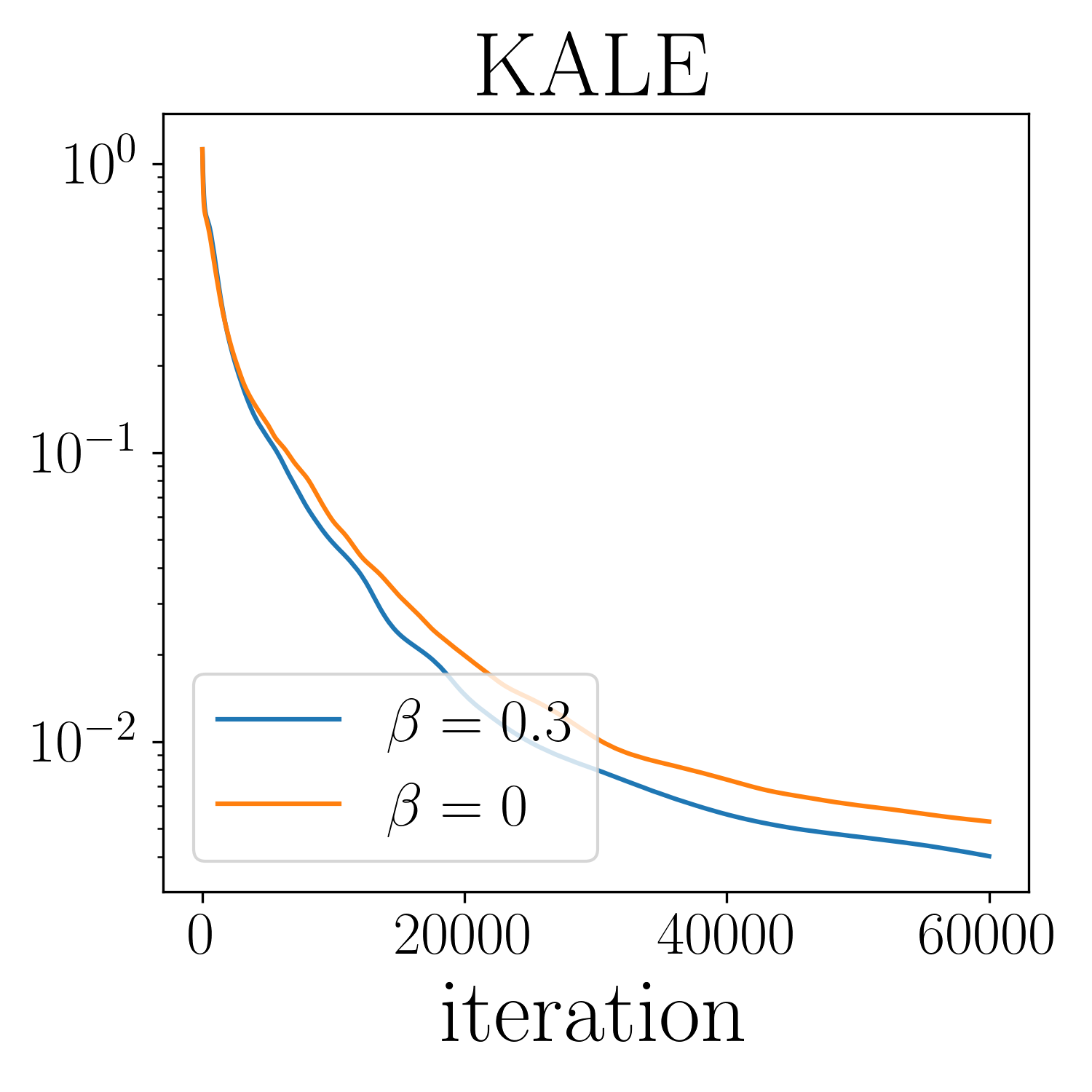} %

    \caption{}\label{fig:KALE-flow-mog-noise-injection}
  \end{subfigure}
\caption{ (a): Evolution of the Wasserstein distance between reference
  particles from the ULA algorithm and the KALE particle descent algorithm with
  various values of $ \lambda $. (b) Left: same as (a), but taking particles from the MMD flow
as reference. (c) Evolution of the KALE along the trajectories of a KALE descent algorithm with the same mixture of Gaussians as target, using $\lambda=0.1$. Orange: without noise injection. Blue: with noise injection
using a constant noise schedule.}
\end{figure*}

{\bf Impact of noise injection} 
On all three examples, using a regularized KALE flow with an appropriately
tuned $ \beta_n $  schedule always improves the proximity to the global
minimum $ \mathbb{ P }_{\infty} = \mathbb{ Q } $. Its effect is particularly
impactful in the mixture of Gaussians example, where a small, constant noise
schedule $ \beta_n $ allows for faster mixing times for $ \mathbb{ P }_n,$ as
opposed to its unregularized counterpart, see
\cref{fig:KALE-flow-mog-noise-injection}. We provide further details on the
impact of noise injection in the appendix.

\section{Discussion and further work}
We have constructed the $ \text{KALE} $ flow, a gradient flow
between probability distributions that relaxes the KL gradient flow for
probabilities with disjoint support. Using the \emph{KALE Particle Descent
Algorithm}, we have shown on several examples that in cases where a KL gradient
flow cannot be defined, trajectories of the KALE flow empirically exhibit
better convergence properties when compared to the MMD flow, a flow that the
KALE is also able to interpolate. In
cases where the KL flow can be defined, we notice empirically that the KALE
flow can \emph{approximate} the trajectories of the KL flow, but using only
information from \emph{samples} of the target. This latter property is in sharp contrast
with KL Gradient Flow discretizations like the Unadjusted Langevin Algorithm:
in this regard, we could use the KALE flow as a sample-based approximation of
the KL flow, which is to our knowledge a novel concept. 
Future work would analyze  when the KALE flow is a consistent estimator of the KL flow in the large sample limit.

\newpage 

\bibliographystyle{abbrvnat}
\bibliography{final_bib}

\clearpage

\clearpage
\appendix
{\large \bf Appendix for }{\large \it KALE flow: A relaxed KL Gradient Flow for Probabilities with Disjoint Support}

The appendix is structured as follows: in \cref{app-sec:background}, we give
additional details on the variational formulation of the KL divergence as well as Wasserstein
gradient flows. 
In \cref{proof:thm-continuity},
\ref{proof:kale-asymptotic-mmd}, we give the proofs for all statements made
about the static properties of the KALE, while in \cref{proof:KALE-diff} to
\ref{proof:kale-descent-vs-kale-flow} we give proofs for all statements made
about the KALE flow and descent algorithm. \cref{app-sec:auxiliary-lemmas}
contains some additional technical lemmas that are used throughout the
appendix. Finally, in \cref{app-sec:numerical-details}, we provide 
details on the experiments discussed in the main body, and the impact of 
noise injection on KALE particle descent trajectories.

\section{Mathematical Background}
\label{app-sec:background}
In this section, we lay out in more depth the theoretical framework behind the
tools used in this paper. We first review the variational formulation of the
KL, and more generally $ f $-divergences.  We discuss how this variational
formulation can be used beyond the context of  statistical estimation of the
KL, which is the original context it was considered for
\cite{nguyen_estimating_2010}. We then provide additional details about
Wasserstein gradient flows, and the theoretical tools used to study them.
\subsection{The use of the variational formulation of $ f $-divergences}
$ f $-divergences, first described in \cite{ali1966general}, form a family of
divergences between probability measures parametrized by a convex, lower semi-continuous function $ f $. 
The divergence $ D_{f} $ between two probabilities measures $ \mathbb{ P } $ and $ \mathbb{ Q } $ is defined as:
\begin{equation*}
\begin{aligned}
  D_{f}(\mathbb{ P } \mid \mid \mathbb{ Q }) = \begin{cases}
    \int_{  }^{  } f( \frac{ \text{d} \mathbb P }{ \text{d} \mathbb Q }) \text{d} \mathbb Q & \text{ if }  \mathbb{ P } \ll \mathbb{ Q }\\
   +\infty & \text{ otherwise}
  \end{cases} 
\end{aligned}
\end{equation*}
Apart from the KL, which we will discuss later, other well known instances of $ f
$-divergences include the  $ \chi^2 $ divergence, the Hellinger divergence and the
Total Variation.  Requiring the function $ f $ to be convex allows to use the
theory of Fenchel duality to frame $ D_{f} $
as the solution of an optimization problem:
\begin{proposition}[{\cite[Lemma 9.4.4]{ambrosio2008gradient}}]
  For any $ \mathbb{ P } $, $ \mathbb{ Q } $ $ \in \mathcal  P( \mathbb{R}^d) $, we have:
  \begin{equation}
  \label{eq:variational-formulation-f-div}
  \begin{aligned}
    D_{f}(\mathbb{ P } \mid \mid \mathbb{ Q }) = \sup_{ h \in C_{b}^{0}(\mathbb{R}^d) } \left \{ \int_{\mathbb{R}^d}^{  } h(x) \text{d} \mathbb P - \int_{  }^{  } {f}^{\star}( h(x)) \text{d} \mathbb Q \right \} 
  \end{aligned}
  \end{equation}
\end{proposition}
Where $ {f}^{\star} $ is the Fenchel convex conjugate \cite{rockafellar_convex_1970} of the convex function $ f $, defined as:
\begin{equation*}
\begin{aligned}
  {f}^{\star}(u) = \sup_{ x \in \mathbb{R}^d } \left \langle u, x \right \rangle  - f(x)
\end{aligned}
\end{equation*}
The $ \text{KL} $ divergence is a particular instance of $ f $-divergence using the pair $ (f, {f}^{\star}) $:
\begin{equation*}
\begin{aligned}
  f(x) = \begin{cases}
    x (\log x - 1) + 1 & \text{ if } x > 0\\
    1&  \text{ if } x = 0 \\
    +\infty & \text{ if } x < 0
  \end{cases} 
  ,\,\,\quad  {f}^{\star}(u) = e^{u} - 1
\end{aligned}
\end{equation*}
\paragraph{M-estimation procedures for $ \text{KL}(\mathbb{ P } \mid \mid \mathbb{ Q }) $} 
The dual formulation in \cref{eq:variational-formulation-f-div} is an
optimization problem with an objective depending on $ \mathbb{ P } $ and $
\mathbb{ Q } $ only through \emph{expectations}. By relying on the theory of
M-estimation,  \cite{nguyen_estimating_2010} showed that it was possible to
\emph{consistently} approximate the population solution of
\cref{eq:variational-formulation-f-div} using only samples $
\{Y^{(i)}\}_{i=1}^{N} $ and $ \{X^{(i)}\}_{i=1}^{N} $ of $ \mathbb{ P } $ and $
\mathbb{ Q } $. In particular, they showed that the solution of the sample-based, regularized problem:
\begin{equation}
\label{eq:kl-estimation-penalized}
\begin{aligned}
  \sup_{ h \in \mathcal  H } \left \{ 1 + \int_{  }^{  } h \text{d} \widehat{ \mathbb{ P } }^{N} - \int_{  }^{
  } e^{h} \text{d} \widehat{ \mathbb{ Q } }^{N} + 1 - \frac{\lambda_N}{2}I(h) \right \}
\end{aligned}
\end{equation}
(where $ I(h) $ is a convex complexity penalty) will converge in probability to
the solution of \cref{eq:variational-formulation-f-div}, provided that $
\lambda_N $ decays to $ 0 $ as $ \frac{1}{\sqrt {N}} $ and that the complexity
of the function class $ \mathcal  H $ is small enough.
However, their setting is general and does not exploit the specificity of an
RKHS $\mathcal H$ with a penalty  $ I(h) = \left \| h \right
\|^2_{\mathcal  H} $.  Consistency for the latter case (a case which is
tightly linked to the definition of the KALE), was proved by
\cite{arbel_generalized_2020_1} using tools from RKHS theory.

\paragraph{Why KALE differs from simple KL estimation} 
The addition of the regularization term $ \frac{ \lambda_N }{ 2 }I(h) $ (where the KALE objective is retrieved using $I(h) = \| h \|^2$) to
\cref{eq:kl-estimation-penalized} makes the solution of
\cref{eq:variational-formulation-f-div} non-infinite for the mutually singular
empirical distributions $ \widehat{ \mathbb{ P } }^{N} $ and $ \widehat{
\mathbb{ Q } }^{N} $. However, the KL population objective
\cref{eq:variational-formulation-f-div} is unregularized, reflecting the fact
that the KL is infinite for mutually singular population $ \mathbb{ P } $ and $
\mathbb{ Q } $. It is the goal of \cref{sec:kale} is to show that extending the
regularization technique introduced in an estimation setting to the
KL population objective results in a relaxed solution to the KL problem that is
a valid divergence measure between $ \mathbb{ P } $ and $ \mathbb{ Q } $. The $
\text{KALE} $ thus leverages the \emph{biases} of the $ \text{KL} $ estimates to remain well-defined
for mutually singular distributions: in the present context, the primary interest of $ \text{KALE} $ is not to
estimate the KL, but to provide a {\em KL alternative} for mutually singular
distributions. This justifies the definition of the KALE with
a positive $ \lambda $ given in \cref{def:kale}. Note that a sample-based approximation of $ \text{KALE}(\mathbb P \mid \mid \mathbb Q) $ is now:

\begin{equation}
\label{eq:kale-estimation-penalized}
\begin{aligned}
  \max_{ h \in \mathcal  H } \left \{1 + \int_{  }^{  } h \text{d} \widehat{ \mathbb{ P } }^{N} - \int_{  }^{
  } e^{h} \text{d} \widehat{ \mathbb{ Q } }^{N} - \frac{\lambda}{2} \| h \|^2 \right \}
\end{aligned}
\end{equation}
We emphasize that unlike in \cref{eq:kl-estimation-penalized}, $ \lambda $ is now kept
fixed.

\subsection{Wasserstein Gradient Flows}

\paragraph{The Wasserstein Geometry}
The theory of Wasserstein-2 gradient flows considers the set of probability
measures on $ \mathcal  P_2(\mathcal  X) $ (where $ \mathcal  X $ is a separable Hilbert Space set to $ \mathbb{R}^d $ in our case) with finite $ 2^{\text{nd}} $ order
moments, endowed with the Wasserstein-2 metric, defined, given $ \mathbb{ P
}_0, \mathbb{ P }_1 \in \mathcal  P_2(\Rd) $, as:

\vspace{-1em}

\begin{equation}
\label{eq:wasserstein}
\begin{aligned}
  W_2(\mathbb{ P }_0, \mathbb{ P }_1) = \left ( \inf_{ \gamma \in 
    \Gamma(\mathbb{ P }_0, \mathbb{ P }_1) § } \int_{  }^{  } \left \| x - y
\right \|^2 \text{d}  \gamma (x, y) \right )^{\frac{1}{2}}
\end{aligned}
\end{equation}
$ \Gamma(\mathbb{ P }_0, \mathbb{ P }_1) $ denotes the sets of \emph{admissible
transport plans} between $ \mathbb{ P }_0$ and $ \mathbb{P}_1 $:
\begin{equation*}
\begin{aligned}
  \Gamma(\mathbb{ P }_0, \mathbb{ P }_1) = \left \{  \gamma \in \mathcal
  P(\Rd \times \Rd); \quad (\pi^{1})_{\#} \gamma = \mathbb{ P
}_0,\,\, (\pi^{2})_{\#} \gamma = \mathbb{ P }_1\right \} 
\end{aligned}
\end{equation*}
where $ \pi^{1}: (x, y)  \longmapsto x $ and $ \pi^{2}  (x, y)  \longmapsto y $
are the canonical projections on $ \Rd \times \Rd $. In the
proofs, we will often consider \emph{constant speed geodesics} between two
probabilities $ \mathbb{ P }_0 $ and $ \mathbb{ P }_1 $, defined as paths $
\left ( \mathbb{ P }_t \right )_{0 \leq t \leq 1} $ of the form: $$ \mathbb{ P
}_t = \left ( (1 - t) \pi^{1} + t \pi^{2} \right )_{\#} \gamma$$ where $ \gamma
\in \Gamma_o(\mathbb{ P }_0, \mathbb{ P }_1) $ is an \emph{optimal coupling},
in the sense that it minimizes the objective defining the $ W_2(\mathbb{ P }_0,  \mathbb{ P }_1) $
distance in \cref{eq:wasserstein}. Convexity along geodesics, or \emph{geodesic
convexity} is a property of functionals in $ (\mathcal  P_2(\Rd), W_2)
$:
\begin{definition}[Geodesic convexity, {\cite[Definition 9.1.1]{ambrosio2008gradient}}]
	We say that a functional $ \mathcal  F $ is $ -M $-geodesically semiconvex for some $ M>0 $ if for any
	$ \mathbb{ P }_0, \mathbb{ P }_1 $ and constant speed geodesic $ \mathbb{ P }_t,\,\, t
		\in \left \lbrack 0, 1 \right \rbrack $ between $ \mathbb{ P
		}_0 $ and $ \mathbb{ P }_1 $, the following holds:
	\begin{equation*}
		\begin{aligned}
		  \mathcal  F(\mathbb{ P }_t) \leq (1 - t) \mathcal  F(\mathbb{
		  P }_0) + t \mathcal  F(\mathbb{ P }_1) + M t \left ( 1 - t
		\right ) W_2(\mathbb{ P }_0, \mathbb{ P }_1)^2.
		\end{aligned}
	\end{equation*}
\end{definition}

\paragraph{Wasserstein Gradient Flows} 
The set $ ( \mathcal  P_2(\mathbb{R}^d), W_2 ) $ is a metric space
and not a Hilbert space. Because of that, the notion of gradient (flow) of a functional $
\mathcal  F $  cannot  easily be defined through duality with the
differential of $ \mathcal  F $, and porting the notion of  ``gradient flow'' to
the space $ (\mathcal  P_2(\Rd), W_2) $ thus requires characterizing gradient
flows trajectories in a Hilbertian-free way. Examples of such characterizations
include curves of maximal slope \cite[Section 11.1.1]{ambrosio2008gradient}, or
identification with limit curves of \emph{minimizing moment schemes}. We refer
to \cite{santambrogio__2016} for an introduction of gradient flows Wasserstein
spaces.
The formal definition of Wasserstein-2 gradient flows as given in \cite{ambrosio2008gradient} is as follows:
\begin{definition}[Gradient Flows {\cite[Definition 11.1.1]{ambrosio2008gradient}}]
  \label{def:formal-gradient-flow}
  We say that an absolutely continuous map  $ \left ( t  \longmapsto \mathbb{ P }_t \in
  \mathcal P_2(\mathbb{R}^d) \right ) $ is a solution of the
  \emph{Wasserstein-2 gradient flow} equation:
  \begin{equation}
  \label{eq:formal-gradient-flow}
  \begin{aligned}
  \partial_t \mathbb{ P }_t + \text{div} \left ( \mathbb{ P }_t v_t \right ) = 0,
  \end{aligned}
  \end{equation}
  if $ \left (  I \times (- v_t) \right ) \in \boldsymbol{\partial} \mathcal  F(\mathbb{ P }_t) $,
  where $ \boldsymbol{\partial}  \mathcal  F(\mathbb{ P }_t) $ is the extended Fr\'{e}chet
  subdifferential of $ \mathcal  F $ evaluated at $ \mathbb{ P }_t$.
\end{definition}

For common functionals such as the sum of (sufficiently smooth) potential, interaction and internal energy terms,
\begin{equation*}
\begin{aligned}
   \mathcal  F(\mathbb{ P }) = \int_{  }^{  } V(x) d \mathbb{ P }(x) + \int_{  }^{  } W(x - y) d \mathbb{ P }(x) d \mathbb{ P }(y) + \int_{  }^{  } f(p(x)) d \mathbb{ P }(x)
\end{aligned}
\end{equation*}
(where $\mathbb{P}$ is assumed to be regular, of the form $\textrm{d}\mathbb{P}(x) = p(x) \textrm{d}x$),
\cite{ambrosio2008gradient} have identified solutions of the very general \cref{eq:formal-gradient-flow} with solutions of the more familiar
\begin{equation}
\label{eq:classical-gradient-flow}
\begin{aligned}
  \partial \mathbb{ P }_t  - \text{div} \left (\mathbb{ P }_t \nabla_{  }   \frac{ \delta \mathcal  F }{ \delta \mathbb{ P } }(\mathbb{ P }_t) \right ) = 0,
\end{aligned}
\end{equation}
where $ \frac{ \delta \mathcal  F }{ \delta \mathbb{ P } } $ is the first variation of $ \mathcal  F $, defined (when it exists) as the function $ v $ verifying:
\begin{equation*}
\label{eq:first-variation}
\begin{aligned}
  \lim_{ \epsilon  \to 0 } \frac{ \mathcal  F(\mathbb{ P } + \epsilon d \chi ) - \mathcal  F(\mathbb{ P }) }{ \epsilon } = \int_{  }^{  } v(x) \text{d} \chi, \quad \chi = \mathbb{ P } - \mathbb{ Q }
\end{aligned}
\end{equation*}
For any $ \mathbb{ Q } \in \mathcal  P_2(\mathbb{R}^d) $.
Note that this case includes the MMD (given regularity assumptions on the
kernel), as discussed in \cite{arbel_maximum_2019_1}, but does not include the
KALE, which is not a functional studied in \cite{ambrosio2008gradient}, and to our knowledge, a
novel object of study in the Wasserstein gradient flow literature. In
\cref{proof:KALE-diff}, we show that the identification between
\cref{eq:formal-gradient-flow} and \cref{eq:classical-gradient-flow} still
holds for the case of KALE, by identifying elements of its (strong)
\emph{extended Fr\'{e}chet subdifferential}.  For completeness, we recall the
definition of a strong extended Fr\'{e}chet subdifferential:
\begin{definition}[(Strong) Extended Fr\'{e}chet subdifferential, {\cite[Definition
  10.3.1]{ambrosio2008gradient}}]\label{def:extended-frechet-subdifferential}
  Let $ \mathcal  F: \mathcal  P_2(\Rd)  \longmapsto \left( 
  -\infty, +\infty \right \rbrack $ be a proper, geodesically convex
  functional that is lower semicontinuous w.r.t $ W_2 $. We say that $ \gamma
  \in \mathcal  P_2(\Rd \times \Rd) $ belongs to the strong extended
  Fr\'{e}chet subdifferential $ \boldsymbol{\partial} F(\mathbb{ P }_0) $ if $ \left ( \pi^{1}
  \right )_{\#} \gamma = \mathbb{ P }_0$, and for every $ \mathbb{ P }_1 \in
  \mathcal  P_2(\Rd) $ and $ \boldsymbol{\mu} \in
  \Gamma(\gamma, \mathbb{ P }_1) $:
	$$ \mathcal  F(\mathbb{ P }_1) - \mathcal  F(\mathbb{ P }_0) \geq \int_{ X^3 }^{  } \left \langle x_2, x_3 - x_1 \right
	  \rangle d \boldsymbol{\mu} + o(W_{2, \boldsymbol{\mu}}(\mathbb{ P }_0, \mathbb{ P }_1))$$
	  where $ W^{2}_{2, \boldsymbol{\mu}}(\mathbb{ P }_0, \mathbb{ P
	  }_1) = \int_{  }^{  } \left \| x_1 - x_3 \right \|^2 d
	  \boldsymbol{\mu}(x_1, x_2, x_3) $.
\end{definition}

\section{Proof of \cref{thm:continuity}} 
Throughout this proof, we will consider the function $ \mathcal  K: \mathcal  H  \times \mathcal  P(\mathbb{R}^d)  \to \mathbb{R} $ given by:
\begin{equation} \label{eq:K-appendix}
\begin{aligned}
  \mathcal  K(h, \mathbb{ P }) = \max_{  } \left \{  1 + \int_{  }^{  } h \text{d} \mathbb P - \int_{  }^{  } e^{h} \text{d} \mathbb Q - \frac{\lambda}{2} \left \|h \right \|^2 \right \}
\end{aligned}
\end{equation}
$ \mathcal  K $ has the same expression as the one of \cref{lemma:link-kale-mmd-witness-function}, with a supercharged signature to include the dependency in $ \mathbb{P} $, which we will use in this proof.
\paragraph{Proof of \cref{lemma:link-kale-mmd-witness-function}} 
The proof follows directly from {\cite[Lemma 8]{arbel_generalized_2020_1}}. By
\cref{assump:bounded-kernel}, all integrability requirements are satisfied
(the two \emph{Bochner} integrals in the next equation are well-defined because of \cref{assump:bounded-kernel}). Following this,
the gradient of $ \mathcal  K $ is given by:
\begin{equation*} \label{eq:kale-obj-grad}
\begin{aligned}
  \nabla_{ h } \mathcal  K(h, \mathbb{ P }) = \int_{  }^{  } k(x, \cdot)
  \text{d} \mathbb P - \int_{  }^{  } k(x, \cdot) e^{h} \text{d} \mathbb Q - \lambda h.
\end{aligned}
\end{equation*}
And its evaluation at $ 0 $ given in \cref{lemma:link-kale-mmd-witness-function} follows.
\qed

\label{proof:thm-continuity}
\paragraph{Proof that the $ \text{KALE} $ is weakly continuous}
Let $ \left ( \mathbb{ P }_n \right )_{n \in \mathbb{N}} $ such that $ \mathbb{ P }_n $
weakly converges to $ \mathbb{ P } $. Let $ {h}^{\star} = \arg \max_{ h } \mathcal  K(h,
	\mathbb{ P }) $ and $ {h}^{\star}_n = \arg\max_{ h } \mathcal K(h, \mathbb{ P }_n) $. We will show that both:
\begin{equation*} 
	\begin{aligned}
		 \limsup\limits_{n  \to \infty} \mathcal  K( {h}^{\star}_n, \mathbb{ P }_n) = \mathcal
		K( {h}^{\star}, \mathbb{ P })                                                            \quad \text{ and } \quad 
		 \liminf\limits_{n  \to \infty} \mathcal  K( {h}^{\star}_n, \mathbb{ P }_n) = \mathcal
		K( {h}^{\star}, \mathbb{ P }).
	\end{aligned}
\end{equation*}
The result on $ \text{KALE} $ follows since when $ \lambda $ is kept fixed, $
	\mathcal  K $ and $ \text{KALE} $ differ only by a multiplicative factor.
	We focus on proving the first ($ \limsup $) equality, the arguments for $ \liminf $ being identical.

First, by optimality of $ {h}^{\star}_n $ w.r.t $ \mathbb{ P }_n $, we have:
$\mathcal  K( {h}^{\star}_n, \mathbb{ P }_n)
\geq \mathcal  K( {h}^{\star}, \mathbb{ P }_n)$, implying
		$$\limsup_{ n  \to \infty } \mathcal  K( {h}^{\star}_n, \mathbb{ P }_n) \geq
		\limsup_{ n  \to \infty } \mathcal K( {h}^{\star}, \mathbb{ P }_n).$$
Since $ \mathbb{ P }_n \rightharpoonup \mathbb{ P } $, the r.h.s verifies
$ \limsup_{n  \to \infty}  \mathcal  K( {h}^{\star}, \mathbb{ P }_n)
		= \lim_{ n \to \infty} \mathcal  K( {h}^{\star}, \mathbb{ P }_n) = \mathcal  K(
		{h}^{\star}, \mathbb{ P })$,
from which we conclude
		$\limsup_{n  \to \infty} \text{KALE}(\mathbb{ P }_n, \mathbb{ Q }) \geq
		\text{KALE}(\mathbb{ P } \mid \mid \mathbb{ Q })$.
To prove the converse, assume that $\limsup_{n  \to \infty} \text{KALE}(\mathbb{ P }_n, \mathbb{ Q }) >
		\text{KALE}(\mathbb{ P } \mid \mid \mathbb{ Q })$. Then there exists $ \epsilon > 0 $ and a subsequence  $n_{k}\rightarrow +\infty$ with $k\rightarrow +\infty$ such that
		$\mathcal  K( {h}^{\star}_{n_{k}}, \mathbb{ P}_{n_k}) \geq \mathcal  K( {h}^{\star},
		\mathbb{ P }) + \frac{\epsilon}{2}$.
Let us now compare $ \mathcal  K( {h}^{\star}_{n_k}, \mathbb{ P }) $ with $
\mathcal  K( {h}^{\star}, \mathbb{ P }) $ :
\begin{equation*} 
	\begin{aligned}
		\mathcal  K( {h}^{\star}_{n_k}, \mathbb{ P }) = \mathcal  K( {h}^{\star}_{n_k}, \mathbb{ P
		}_{n_k}) +  \int_{  }^{  } {h}^{\star}_{n_k} \text{d} (\mathbb P  - \mathbb P_{n_k})
		\geq \mathcal  K( {h}^{\star}, \mathbb{ P }) + \frac{\epsilon}{2} - \frac{ 4 \sqrt {K}\text{MMD}
		( \mathbb{ P } \mid \mid \mathbb{ P }_{n_k} )  }{ \lambda }
	\end{aligned}
\end{equation*}
where for the last step, we used the Cauchy-Schwarz inequality and
\cref{lemma:bounded-kale-witness-function}.  Since the MMD is weakly
continuous for bounded kernels with Lipschitz embeddings {\cite[Theorem 3.2]{bharath_on_the_optimal}},  we have $ \lim_{ k  \to \infty
}\text{MMD}^2(\mathbb{ P }_{n_k} \mid \mid \mathbb{ P })  = 0$: there exists a
$ k_0 $ such that, for $ k > k_0 $, 
		$\mathcal  K( {h}^{\star}_{n_k}, \mathbb{ P }) > \mathcal  K( {h}^{\star},
		\mathbb{ P }) + \frac{\epsilon}{4}$, 
which contradicts the optimality condition defining $ {h}^{\star} $. Hence, we must have
\[
	\limsup\limits_{n  \to \infty} \text{KALE}(\mathbb{ P }_n \mid \mid \mathbb{ Q }) =
	\text{KALE}(\mathbb{ P } \mid \mid \mathbb{ Q }).
\]
The two steps of this proof can be repeated for any convergent subsequence of $
\mathcal K({h}^{\star}_n, \mathbb{ P }_n)$, and as a consequence, we also have:
$\liminf\limits_{n  \to \infty} \text{KALE}(\mathbb{ P }_n \mid \mid \mathbb{ Q
}) = \text{KALE}(\mathbb{ P } \mid \mid \mathbb{ Q })$, which proves the weak
continuity of KALE.
\qed

\paragraph{Proof that $ \text{KALE} $ is a probability divergence that metrizes
the weak convergence of probability distributions}

We first prove positivity and definiteness of KALE, making it a probability
divergence.  Positivity of KALE comes from the fact that $ \mathcal  K(
{h}^{\star}, \mathbb{ P }) \geq \mathcal  K(0, \mathbb{ P }) = 0
$. To prove definiteness of KALE, assume $ \text{KALE}(\mathbb {P} \mid \mid
\mathbb {Q}) = 0 $.
Recall that $ \text{KALE}(\mathbb{ P } \mid \mid \mathbb{ Q }) = 0 \iff {h}^{\star} =
	0$, since $ \mathcal  K( 0, \mathbb{ P }) = 0 $ and the
objective is strongly convex.  The optimality criterion for $ 0_{\mathcal  H} $
can be characterized by differentiating
$ \mathcal  K( h, \mathbb{ P }) $. Using \cref{lemma:link-kale-mmd-witness-function}, and the optimality of $ 0 $, we have:
\begin{equation*} 
	\begin{aligned}
		0 = \nabla_{ h }  \mathcal  K(0, \mathbb{ P }) \overset{\Delta}{=} \int_{  }^{  } k(x, \cdot) \text{d} \mathbb P -
		\int_{  }^{  } k(x, \cdot) \text{d} \mathbb Q = f_{\mathbb{ P }, \mathbb{ Q }},
	\end{aligned}
\end{equation*}
where $ f_{\mathbb{ P }, \mathbb{ Q }} $ denotes the MMD \emph{witness function} between $
	\mathbb{ P } $ and $ \mathbb{ Q } $, i.e. $ \text{MMD}(\mathbb{ P } \mid \mid
	\mathbb{ Q })^2 = \left \|f_{\mathbb{ P }, \mathbb{ Q }} \right \|^2 $. When $ k $ is
universal,  $ f_{\mathbb{ P }, \mathbb{ Q }} = 0 $ is only possible when $ \mathbb{ P }
	= \mathbb{ Q }$, which proves the first implication of the equivalence. The reverse
implication is proven by noticing that
\begin{equation*} 
	\begin{aligned}
		\mathbb{ P } = \mathbb{ Q } \Longrightarrow
		\nabla_{ h } \mathcal  K( 0, \mathbb{ P }) = 0.
	\end{aligned}
\end{equation*}
\textbf{Metrizing weak convergence}
From the weak continuity of KALE associated with the definiteness of KALE proven above, we have $
\mathbb{ P }_n \rightharpoonup \mathbb{ Q } \Longrightarrow \text{KALE}
(\mathbb{ P }_n \mid \mid \mathbb{ Q })  \to 0 $.
For the converse, assume that $ \text{MMD}(\mathbb{P}_n \mid \mid \mathbb{Q})$ doesn't converge to $0$. Therefore, there exists a subsequence  $n_k$ with $n_k\rightarrow +\infty$ when $k\rightarrow +\infty$ and such that $ \text{MMD}(\mathbb{P}_{n_k} \mid \mid \mathbb{Q}) > c>0$ for some $c>0$.   Fix $ \epsilon > 0 $. We have that:

\begin{equation*} 
	\begin{aligned}
		\text{KALE}(\mathbb{ P }_{n_k} \mid \mid \mathbb{ Q }) & \geq  \mathcal  K( \epsilon \times
		f_{\mathbb{ P }_{n_n}, \mathbb{ Q }} ) = \left \langle \nabla_{ h } \mathcal  K(
		0,
		\mathbb{ P }_{n_k}), \epsilon f_{\mathbb{ P }_{n_k}, \mathbb{ Q }} \right \rangle         
		                                          + \mathcal  O ( \epsilon^2 
		\|f_{\mathbb{ P }_{n_k}, \mathbb{ Q }}  \|)                                                                           \\
		                                          & = \epsilon \|f_{\mathbb{ P }_{n_k},
		\mathbb{ Q }}  \| + \mathcal
		O(\epsilon^2  \|f_{\mathbb{ P }_{n_k},
		\mathbb{ Q }}  \|^2).                                \\
	\end{aligned}
\end{equation*}
Now, recall that $\Vert f_{\mathbb{ P }_{n_k}, \mathbb{ Q }} \Vert =
\text{MMD}(\mathbb{ P }_{n_k} \mid \mid \mathbb{ Q }) \geq  c > 0$, 
implying that for sufficiently low $ \epsilon $, we will have:
		$\text{KALE}(\mathbb{ P }_{n_k} \mid \mid \mathbb{ Q }) > \frac{c\epsilon}{2},\,\,\, \forall k \geq  n_k$.
		Thus, $ \text{KALE}(\mathbb{ P }_n \mid \mid \mathbb{ Q }) $
		does not tend to 0. Hence, by contradiction $ \text{MMD}(\mathbb{P}_n
		\mid \mid \mathbb{Q})$ converges to $0$ which implies that
		$\mathbb{P}_n $ converges weakly to $\mathbb{Q}$ since the MMD
		metrizes weak convergence. This concludes the proof of
		\cref{thm:continuity}.
		\qed

\section{Proof of \cref{prop:kale-asymptotic-mmd}}
\label{proof:kale-asymptotic-mmd}
\subsection{Proof of (i)} 
To prove that KALE converges to the MMD as $ \lambda $ increases, we will show
the following inequalities:
\begin{equation*} 
	\begin{aligned}
		\frac{1}{2}\text{MMD}^2(\mathbb{ P } \mid \mid \mathbb{ Q }) - \mathcal  O\left(\frac{1}{\lambda}\right)
		\leq \text{KALE}(\mathbb{ P } \mid \mid \mathbb{ Q })
		\leq  \frac{1}{2} \text{MMD}^2(\mathbb{ P } \mid \mid \mathbb{
		Q }) + \mathcal  O\left(\frac{1}{\lambda}\right)
	\end{aligned}
\end{equation*}
To prove the right inequality, we recall that $ \mathcal  K(h, \mathbb{ P })
\leq \int_{  }^{  } h \text{d} \mathbb P - \int_{  }^{  } h \text{d} \mathbb Q
- \frac{\lambda}{2} \left \| h \right \|^2$, which holds by
convexity of the exponential. The right-hand side is maximized for $
{h}^{\star} = \frac{ f_{\mathbb{ P }, \mathbb{ Q}} }{ \lambda } $ and equals $
\frac{ \text{MMD}^2(\mathbb{ P } \mid \mid \mathbb{ Q })}{2 \lambda} $.
Consequently, we have: $ \text{KALE}(\mathbb{ P } \mid \mid \mathbb{ Q }) \leq \frac{
1 + \lambda }{ 2 \lambda} \text{MMD}^2( \mathbb{ P } \mid \mid \mathbb{ Q }) $.

To prove the left inequality, we use \cref{lemma:bounded-kale-witness-function}
which allows to control the discrepancy between the KALE and the MMD. Indeed, we have:
$ h(x) = \left \langle h, k(x, \cdot) \right \rangle \leq \sqrt {K}  \left \|h \right \| = \frac{ 4 K }{ \lambda } $.
The following Taylor-Lagrange inequality holds, uniformly for all $ x $:
\begin{equation*} 
	\begin{aligned}
	  e^{h(x)} & \leq 1 + h(x) + \frac{ e^{ \frac{ 4K }{ \lambda }} 16 K^2 }{ 2 \lambda^2},
	\end{aligned}
\end{equation*}
which gives a lower bound of $ \mathcal  K( {h}^{\star}, \mathbb{ P }) $:
\begin{equation*} 
	\begin{aligned}
	  \left ( 1 + \lambda \right ) \mathcal  K( h, \mathbb{ P }) \geq (1 + \lambda)
	  \left ( \int_{  }^{  }h \text{d} \mathbb P - \int_{  }^{  } h \text{d} \mathbb
	  Q - \frac{\lambda}{2} \left \|h \right \|^2  - \frac{ 8 K^2 e^{ \frac{ 4K }{ \lambda }} }{ \lambda^2}\right ).
	\end{aligned}
\end{equation*}
Remark that the r.h.s is maximized for $ h_1 = f_{\mathbb{ P }, \mathbb{
  Q}}/{ \lambda } $. Because $ {h}^{\star} $ maximizes the l.h.s, we have:
\begin{equation*} 
	\begin{aligned}
	  \left ( 1 + \lambda \right ) \mathcal  K( {h}^{\star}, \mathbb{ P }) \geq (1 + \lambda) \mathcal  K(h_1, \mathbb{ P }) \geq \frac{ 1 + \lambda }{
			2 \lambda} \text{MMD}^2( \mathbb{ P } \mid \mid \mathbb{ Q }) - \frac{ 8 K^2 e^{ \frac{ 4K }{ \lambda }}(1 +
			\lambda) }{ \lambda^2 }.
	\end{aligned}
\end{equation*}
The two initial inequalities are verified, and taking them to the limit $ \lambda  \to \infty $ concludes the proof.
\qed

\subsection{Proof of (ii)}
(ii) was proved in \cite{arbel_generalized_2020_1} as part of (Theorem 7). For completeness, we recall the elements of the proof. 
Let us highlight the dependency of $ {h}^{\star} = \arg \max_{ h } \mathcal  K(h, \mathbb{ P }) $ in $ \lambda $ (see \cref{eq:K-appendix}) by noting it ${h}^{\star}_{\lambda} ({=} {h}^{\star}) $, for $ \lambda \geq  0 $.
Because we assume that $ \log \frac{ \text{d} \mathbb P }{ \text{d} \mathbb Q }
\in \mathcal  H $, we have:
\begin{equation*} 
\begin{aligned}
  {h}^{\star}_0 = \log \frac{ \text{d} \mathbb P }{ \text{d}\mathbb Q }, \quad 
  \text{KL}(\mathbb{ P } \mid \mid \mathbb{ Q }) = 1 + \int_{  }^{  } {h}^{\star}_0 \text{d} \mathbb P - \int_{  }^{  } e^{ {h}^{\star}_0} \text{d} \mathbb Q.
\end{aligned}
\end{equation*}

Thus, we have
\begin{equation*} 
\begin{aligned}
  \Big \lvert  \frac{ \text{KALE}(\mathbb{ P } \mid \mid \mathbb{ Q }) }{ (1 + \lambda) } - \text{KL}(\mathbb{ P } \mid \mid \mathbb{ Q }) \Big \rvert  &= \Big \lvert 1 + \int_{  }^{  } {h}^{\star}_{\lambda} \text{d} \mathbb P - \int_{  }^{  } e^{
  {h}^{\star}_{\lambda}} \text{d} \mathbb Q - \frac{\lambda}{2} \left \|
  {h}^{\star}_{\lambda} \right \|^2_{ \mathcal  H} - \text{KL}(\mathbb{ P } \mid \mid \mathbb{ Q }) \Big \rvert \\
															      &=\Big \lvert \int_{  }^{  }
  \left ( {h}^{\star}_{\lambda} - {h}^{\star}_{0} \right
    )\text{d} \mathbb P - \int_{  }^{  } e^{h_0}(1 - e^{({h}^{\star}_{\lambda} -
  {h}^{\star}_0)})\text{d} \mathbb Q + \frac{\lambda}{2} \left \| {h}^{\star}_{\lambda} \right \|^2 \Big \rvert \\
															      &\leq \Big \lvert \int_{  }^{  }
  \left ( {h}^{\star}_{\lambda} - {h}^{\star}_{0} \right
    )\text{d} \mathbb P \Big \rvert + \Big \lvert \int_{  }^{  } e^{h_0}(1 - e^{({h}^{\star}_{\lambda} -
{h}^{\star}_0)})\text{d} \mathbb Q \Big \rvert + \Big \lvert \frac{\lambda}{2} \left \| {h}^{\star}_{\lambda} \right \|^2 \Big \rvert
\end{aligned}
\end{equation*}

To bound the last term, we note that
\begin{equation} 
\label{eq:h-lambda-vs-h0}
\begin{aligned}
\left \| {h}^{\star}_{\lambda} \right \| \leq \left \| {h}^{\star}_{0} \right \|.
\end{aligned}
\end{equation}
Otherwise, by optimality of $ {h}^{\star}_0 $, we have:
\begin{equation*} 
\begin{aligned}
\int_{   }^{  } {h}^{\star}_{\lambda} \text{d} \mathbb P - \int_{  }^{  } e^{ {h}^{\star}_{\lambda}} \text{d} \mathbb Q \leq \int_{  }^{  } {h}^{\star}_{0} \text{d} \mathbb P - \int_{   }^{  } e^{ {h}^{\star}_{0}} \text{d} \mathbb Q \\
\Longrightarrow
\int_{   }^{  } {h}^{\star}_{\lambda} \text{d} \mathbb P - \int_{  }^{  } e^{ {h}^{\star}_{\lambda}} \text{d} \mathbb Q  - \frac{\lambda}{2} \left \| {h}^{\star} \right \|^2_{\mathcal  H}\leq \int_{  }^{  } {h}^{\star}_{0} \text{d} \mathbb P - \int_{   }^{  } e^{ {h}^{\star}_{0}} \text{d} \mathbb Q  - \frac{\lambda}{2} \left \| {h}^{\star}_{0} \right \|^2_{\mathcal  H},\\
\end{aligned}
\end{equation*}
contradicting the optimality of $ {h}^{\star}_{\lambda} $.
As a consequence,  we have that $ \lim_{ \lambda  \to 0 }\frac{\lambda}{2}
\| {h}^{\star}_{\lambda} \|^2 ( \leq \frac{\lambda}{2}
\| {h}^{\star}_0 \|^2) = 0 $.
To bound the first two terms, we use \cite{arbel_generalized_2020_1} (Lemma 11), ensuring that:
\begin{equation} \label{eq:hlambda-h0}
\begin{aligned}
\lim_{  \lambda  \to 0 }\left \| {h}^{\star}_{\lambda} - {h}^{\star}_{0} \right
\| = 0.
\end{aligned}
\end{equation}
As a consequence:
\begin{itemize}
  \item For all $ x \in \Rd$, $ \lim_{  \lambda  \to 0
    }{h}^{\star}_{\lambda}(x) - {h}^{\star}_0(x) = 0 $.
  \item $  {h}^{\star}_{\lambda} $ is a bounded function.
\end{itemize}
We conclude that the first two terms tend to 0 as $ \lambda  \to 0 $
by the dominated convergence theorem.
We thus have:
$\lim_{ \lambda  \to 0 }\Big \lvert \text{KALE}(\mathbb{ P } \mid \mid \mathbb{ Q }) - \text{KL}(\mathbb{ P } \mid \mid \mathbb{ Q }) \Big \rvert = 0$.
\qed

\section{Proof of \cref{prop:KALE-gf}}\label{proof:KALE-diff}
As explained in the introduction, the Wasserstein gradient flow of the KALE does
not have a known expression, other than the abstract one given by
\cref{def:formal-gradient-flow}, applied to the KALE. Relying on the formalism
introduced in \cite{ambrosio2008gradient}, we first show that KALE's gradient
flow admits the ``traditional'' form:
\begin{equation*} 
\begin{aligned}
  \partial_t \mathbb{ P }_t - \text{div} \left ( \mathbb{ P }_t \nabla_{  } \frac{ \delta \text{KALE} }{ \delta \mathbb{ P } } \right ) = 0
\end{aligned}
\end{equation*}
We start by giving an expression of the \emph{first variation} of the $ \text{KALE} $. This
proof is the first in the appendix that involves an implicit function theorem
argument, which we justify at length. For brevity, the same justifications will
be skipped in other proofs relying on small variations around the same implicit
function theorem argument.
\begin{lemma}[Differentiability of KALE]
    \label{lemma:KALE-diff}$ \quad $ \newline
    Let $ \mathbb{ Q } \in \mathcal  P_2(\mathbb{R}^d) $, and $ \lambda > 0 $. Then, the function $ \mathbb{ P } \in \mathcal  P_2(\mathbb{R}^d)  \longmapsto \text{KALE}(\mathbb{ P } \mid \mid \mathbb{
    Q }) $ is G\^{a}teaux differentiable w.r.t. $ \mathbb{ P } $ and admits the following
    first variation:
    \begin{equation*}
    \begin{aligned}
      \frac{ \delta \text{KALE}(\mathbb{ P } \mid \mid \mathbb{ Q }) }{ \delta \mathbb{
      P} }  = \left ( 1 + \lambda \right ){h}^{\star}, \quad {h}^{\star}  = \arg \max_{ h \in \mathcal  H} \mathcal  K(h, \mathbb{ P }).
    \end{aligned}
    \end{equation*}
\end{lemma}
\begin{proof}
  Informally, computing the first variation of $ \text{KALE} $ w.r.t $ \mathbb{
  P } $ can be done using a chain rule argument:
  \begin{equation*} 
  \begin{aligned}
    \frac{ \delta \text{KALE} }{ \delta \mathbb{ P } } = \frac{ \delta 
    \text{KALE}( {h}^{\star}(\mathbb{ P }), \mathbb{ P }) }{  \delta \mathbb{ P }} = \frac{
  \partial \text{KALE} }{ \partial \mathbb{ P } } + \frac{ \partial  \text{KALE}
}{ \partial h }\Big \lvert_{ {h}^{\star}} \frac{ \partial {h}^{\star} }{ \partial \mathbb{
P } } = \frac{ \partial \text{KALE} }{ \partial \mathbb{ P } }
  \end{aligned}
  \end{equation*}
  where the second term is 0 given that $ {h}^{\star} $ is defined as $ \max_{
  h \in \mathcal  H} \mathcal  K(h, \mathbb{ P }) $. To make this discussion rigorous, we need
  to make sure that ``$ \smash{\frac{ \partial {h}^{\star} }{ \partial \mathbb{ P }}
  } $'' (formally, the G\^{a}teaux derivative of the map $ \smash{\mathbb{ P }   \longmapsto {h}^{\star}(\mathbb{ P }) }$)
  exists.

  We recall that given two topologically convex vector spaces $ X $ and  $ Y $,
  and a function $ f: X  \to Y $, the G\^{a}teaux derivative of $ f $ at $ x
  $ in the direction $ \chi  \in X$ is defined as:
  \begin{equation*} 
  \begin{aligned}
    D f(x; \chi) = \lim_{ t  \to 0 } \frac{ f(x + t \chi) - f(x) }{ t }.
  \end{aligned}
  \end{equation*}
    A complete argument for the differentiability of both $ \mathcal  K $
  and $ \mathbb{ P }  \longmapsto {h}^{\star}(\mathbb{ P }) $ would require
  augmenting the domains of functionals of interest from $ \mathcal  P_2(\Rd)
  $ (which is not a vector space) by the vector space of signed radon measures $
  \mathcal M(\Rd) $.  We circumvent this additional
  step by simply considering ``admissible'' directions $ \chi $, such that $
  \int_{  }^{  } d \chi = 0 $. Noting $ {h}^{\star}_t = \arg \max_{ h }
  \mathcal  K(h, \mathbb{ P } + t \chi) $, we know given
  \cref{lemma:link-kale-mmd-witness-function} that $ {h}^{\star}_t $ verifies:
  \begin{equation*} 
	  \begin{aligned}
	    \mathcal  F_{\chi}( {h}^{\star}_t, t) \overset{\Delta}{=} \nabla_{ h } \mathcal  K( {h}^{\star}_t, \mathbb{ P }) = \int\limits_{  }^{
	    } k(x, \cdot)\text{d} (\mathbb P(x) + t \chi(x))
		  - \int\limits_{  }^{  }k(x, \cdot) \exp \left ( {h}^{\star}_t(x)\right ) \text{d} \mathbb Q(x)
		  -  \lambda {h}^{\star}_t = 0.
	  \end{aligned}
  \end{equation*}

  Thus, $ {h}^{\star}_t $ is defined \emph{implicitly} through $
  \mathcal  K $'s optimality at $ {h}^{\star}_t $. To study the differentiability of the mapping $ t  \longmapsto {h}^{\star}_t $, it is natural to rely on an implicit function theorem
  argument on $ \mathcal  F: \left ( \mathcal  H \times \mathbb{R} \right )
  \to \mathcal  H$.  Similarly to implicit function theorems on euclidean
  spaces, we will need to invert $ D_h \mathcal  F_{\chi}(h, t)$, the
  (Fr\'{e}chet) differential of $ \mathcal  F $ w.r.t $ h $. This differential
  is given by:
\begin{equation*} 
	\begin{aligned}
		D_h \mathcal  F_{\chi}(h, t) & = - \underbrace{\int\limits_{  }^{  } k(x, \cdot) \otimes k(x, \cdot)
		e^{h(x)} \text{d} \mathbb Q(x)}\limits_{ \overset{\Delta}{=} \boldsymbol{L}(h)}
		- \lambda \text{I}    ,                                                                             \\
	\end{aligned}
\end{equation*}
which is an invertible operator on $ \mathcal  H $, given that $ \boldsymbol{L}(h) $ is self-adjoint and positive for all $ h $.
We can now apply an implicit function theorem on Banach spaces \cite{lang_fundamentals_1999}
(Theorem 5.9): For all $ \chi $, there exists a neighborhood of $ 0
$,   $ \mathcal  V(0) $, such that the mapping $ t \in \mathcal  V(0)  \longmapsto {h}^{\star}_t $
is differentiable. The derivative of $ {h}_t^{\star} $ at 0 is then the
G\^{a}teaux derivative of $ {h}^{\star}(\mathbb{ P })  $ in the direction $ \chi $:
\begin{equation*} 
\begin{aligned}
  D_{\mathbb{ P }} {h}^{\star}(\mathbb{ P }; \chi) = \int_{  }^{  } (\boldsymbol{L}( {h}^{\star}) + \lambda I )^{-1} k(x, \cdot) d \chi.
\end{aligned}
\end{equation*}
To conclude on $ \text{KALE} $'s first variation, we can rigorously write, using
the chain rule of G\^{a}teaux derivatives,
\begin{equation*} \label{eq1}
	\begin{aligned}
		D_{\mathbb{ P }} \text{KALE}(\mathbb{ P } \mid \mid \mathbb{ Q
		  }; \chi) & = \int\limits_{  }^{  } {h}^{\star}(\mathbb{
			P })(x) d \chi(x) +                                                                               
		                                                    \langle \underbrace{ \nabla_{ h }
			\mathcal  K( {h}^{\star}(\mathbb{ P }), \mathbb{ P })
		      }\limits_{=0},  D {h}^{\star}(\mathbb{ P };
		      \chi) \rangle_{\mathcal  H}                                    
		                                                    = \int\limits_{  }^{  }
		{h}^{\star}(\mathbb{ P }) d\chi                                                                   \\
	\end{aligned}
\end{equation*}
which concludes the proof.
\end{proof}

We now show that the KALE admits strong Fr\'{e}chet subgradients, and that they are
equal to the gradient of KALE's first variation.

\begin{lemma}
  A coupling  $ \gamma $ of the form $ (I \times v)_{\#} \mathbb{ P }_0 $
  belongs to the extended (strong) Fr\'{e}chet subdifferential of $ \text{KALE} $ at $
  \mathbb{ P } = \mathbb{ P }_0 $ if and only if $ v = \nabla_{  }   \frac{ \delta \text{KALE} }{ \delta
  \mathbb{ P } } = (1 + \lambda) \nabla_{  }{h}^{\star}_0 \quad \mathbb{ P
}_0\text{-a.e}$, where $ {h}^{\star}_0 = \arg \max_{
h } \mathcal K(h, \mathbb{ P }_0) $ is the first variation of $ \text{KALE} $
at $ \mathbb{ P } = \mathbb{ P }_0 $.
\end{lemma}
\begin{proof}
  Using an analogue of \cite[Equation 10.3.13]{ambrosio2008gradient} for the
  extended \emph{strong} Fr\'{e}chet subdifferential, we have that:
  \begin{equation*} 
  \begin{aligned}
    \gamma = (I \times v)_{\#} \mathbb{ P }_0 \in \boldsymbol{\partial} \text{KALE}(\mathbb{ P }_0 \mid \mid \mathbb{ Q }) &\iff  \\
      \text{KALE}(\mathbb{ P }_1 \mid \mid \mathbb{ Q }) - \text{KALE}(\mathbb{
      P }_0 \mid \mid \mathbb{ Q })\geq \int_{  }^{  } (y - x)^{\top} & v(x)
      \text{d} \tilde{\gamma}(x, y) + o (C_2( \tilde{ \gamma }))
  \end{aligned}
  \end{equation*}
  for any  $ \mathbb{ P }_1 \in \mathcal  P_2(\Rd) $, $ \tilde{ \gamma } \in
  \Gamma(\mathbb{ P }_0, \mathbb{ P }_1) $.  Note that without loss of
  generality, we switched the coupling $ \boldsymbol \mu \in \Gamma((I \times
  v)_\# \mathbb{ P }_0, \mathbb{ P }_1) $ present in
  \cref{def:extended-frechet-subdifferential} with a coupling $ \tilde{ \gamma
  } \in \Gamma(\mathbb{ P }_0, \mathbb{ P }_1)$, a switch that is made possible
  because of the specific form of $ \gamma $ considered above, which is the one
  needed in \cref{def:formal-gradient-flow}.
  Our goal is to show that $ (I \times v)_{\#} \mathbb{ P }_0 \in \boldsymbol{\partial} \text{KALE}(\mathbb{ P
  }_0 \mid \mid \mathbb{ Q }) \iff v = (1 + \lambda) \nabla_{  } {h}^{\star}_0 $.
  \newline
  We first show the reverse implication, e.g. $ (I \times (1 + \lambda) \nabla_{
  } {h}^{\star}_0)_{\#}\mathbb{ P }_0 \in \boldsymbol {\partial}\text{KALE}(\mathbb{ P }_0 \mid
  \mid \mathbb{ Q })  $. To do so, we consider the following interpolation
  scheme between $ \mathbb{ P }_0 $ and
  $ \mathbb{ P }_1 $:
  \begin{equation*} 
  \begin{aligned}
    \mathbb{ P }_t = \left ( t \pi^{2} + (1 - t) \pi^{1} \right )_{\#} \tilde{\gamma}.
  \end{aligned}
  \end{equation*}
  And note for each $ \mathbb{ P }_t $, $ {h}^{\star}_t = \arg \max_{ h }
  \mathcal  K(h, \mathbb{ P }_t) $.
  Noting $ g(t) = \text{KALE}(\mathbb{ P }_t \mid \mid \mathbb{ Q }) $, we have:
  \begin{equation*} 
  \begin{aligned}
    g'(t) = (1 + \lambda)\int_{  }^{  } \left ( y  - x\right )^{\top} \nabla_{  }{h}^{\star}_t( ty + (1 -t)x)  \text{d} \tilde{\gamma}(x, y), \quad g''(t)= (1 + \lambda) ((I) + (II))
  \end{aligned}
  \end{equation*}
  where
  \begin{equation*} 
  \begin{aligned}
    (I) &= \int_{  }^{  } \left ( y - x \right )^{\top}  \left ( \textbf{H} {h}^{\star}_t(ty + (1 - t)x) (y - x) \right ) \text{d} \tilde{\gamma}(x, y) \\
  (II) &=  \int_{  }^{  } \left ( y - x \right )^{\top}  \left ( \nabla_{  }   \frac{ \text{d} {h}^{\star}_t }{ \text{d}t }(ty + (1 - t) x) \right ) \text{d} \tilde{\gamma}(x, y)
  \end{aligned}
  \end{equation*}
  (and we exchanged the $ t $-derivative and $ \nabla_{  }$ in (II)). From
  \cref{assump:smooth-kernel} we have that $ \left \| \boldsymbol{H} h \right \|
  \leq \left \| h \right \| \sqrt {K_{2d}} \leq \frac{ 4 \sqrt {K K_{2d}}  }{ \lambda } $,
  implying $(I) \leq \frac{ 4 \sqrt {K K_{2d}}}{ \lambda } C_2^{2}(\tilde{\gamma})$.
  Using an implicit function theorem argument, we have:
  \begin{equation*} 
  \begin{aligned}
    \frac{ \text{d} {h}^{\star}_t }{ \text{d}t } = - ( \boldsymbol{L}( {h}^{\star}_t) + \lambda I)^{-1} (y - x)^{\top} \nabla_{ 1 } k_{ty + (1 - t)x},
  \end{aligned}
  \end{equation*}
  implying
  \begin{equation*} 
  \begin{aligned}
	(II)	&=  \int_{  }^{  } \left \langle  \sum\limits_{ i=1 }^{ d } (y_i - x_i) \partial_i k_{ty + (1 - t)x} , \sum\limits_{ i=1 }^{ d } (y_i - x_i) (\textbf{L}( {h}^{\star}_t) + \lambda I)^{-1} \partial_i k_{ty + (1 - t)x}\right \rangle d \tilde{\gamma}(x, y) \\
	  &\leq \frac{ K_{1d} }{ \lambda } C_2^{2}(\tilde{\gamma})
  \end{aligned}
  \end{equation*}
  where the last line was obtained using the Cauchy-Schwarz inequality on $ \mathcal  H $,
  RKHS norm homogeneity, the $ \frac{1}{\lambda} $-bound on $ \|(L + \lambda
  I)^{-1}\| $, and then the Cauchy-Schwarz inequality on $ \mathbb{R}^d $. Using now  Taylor's
  inequality upper bounding the second derivative of $ g $ between $ t = 0 $
  and  $ t = 1 $, we have that:
  \begin{equation*} 
  \begin{aligned}
    g(1) - g(0) \geq \int_{  }^{  } (y - x)^{\top} \nabla_{  } {h}^{\star}_0(x)
    \text{d} \tilde{\gamma}(x, y) + \mathcal  O(C_2^{2}(\tilde{\gamma})).
  \end{aligned}
  \end{equation*}
    Since $ \mathcal  O(C_2^{2}( \tilde{\gamma})) = o(C_2( \tilde{\gamma})) $, it follows that $ (I \times (1 + \lambda)\nabla_{  } {h}^{\star}_0)_{\#}\mathbb{ P }_0 \in \boldsymbol{\partial}\text{KALE}(\mathbb{ P }_0 \mid \mid \mathbb{ Q }) $. \newline
  To prove the reverse implication,
  assume $ v (\overset{\Delta}{=} (1 + \lambda) \tilde{ v }) \ne (1 + \lambda)\nabla_{  }   {h}^{\star}_0 $. Fix $ u > 0 $, and choose an
  ``adversarial'' $ \mathbb{ P }_{1, u} $ defined as $ \mathbb{ P }_{1, u} = (x
  \longmapsto x  + u(1 + \lambda) ( \tilde{ v }(x) - \nabla_{  } {h}^{\star}_0(x) ))_{\#}\mathbb{ P }_0
  $, with an associated coupling $ \tilde{\gamma} = (x \times (x  \longmapsto x + (1 + \lambda) u( \tilde{ v }(x) -
  \nabla_{  } {h}^{\star}_0(x) ))_{\#}\mathbb{ P }_0 $. We then have, using a
  Taylor inequality \emph{lower bounding} the second derivative of $ g $:
  \begin{equation*} 
  \begin{aligned}
    g(1) - g(0)  - \int_{  }^{  } (y - x)^{\top} (1 + \lambda) \tilde{ v }(x) d \tilde{\gamma}(x, y) &\leq \int_{  }^{  } (y - x)^{\top} (1 + \lambda)(\nabla_{  } {h}^{\star}_0(x) - \tilde{ v }(x))  d \tilde{\gamma}(x, y) \\
												     &  + \mathcal  O(C_2^{2}( \tilde{\gamma})) \\
     &\leq -u (1 + \lambda) \int_{  }^{  }  \left \| \tilde{ v }(x) - \nabla_{  } {h}^{\star}_0(x))\right \|^2  \text{d} \mathbb{ P }_0(x)\\
									     & + \mathcal  O(C_2^{2}( \tilde{ \gamma })).
  \end{aligned}
  \end{equation*}
  In the limit  $ \mathbb{ P }_{1, u} \rightharpoonup \mathbb{ P }_0 $, e.g. $ u  \to 0 $),
  the right-hand side scales in $ u $, which is the same scaling as  $ C_2( \tilde{\gamma}) = (\int_{  }^{
    } \left \| x_1 - x_2 \right \|^2 d \tilde{\gamma}(x_1, x_2))^{1 / 2} = u ( 1 + \lambda) (\int_{  }^{
} \left \|  \tilde{ v }(x) - \nabla_{  } {h}^{\star}_0(x)  \right \|^2 d \mathbb{ P
  }_0(x))^{1 / 2} $.
  Thus, it follows that the inequality:
  \begin{equation*} 
  \begin{aligned}
    g(1) - g(0)  - \int_{  }^{  } (y - x)^{\top} (1 + \lambda)  \tilde{ v }(x) d \tilde{\gamma}(x, y)
    &\geq  o(C_2( \tilde{\gamma})) \\
  \end{aligned}
  \end{equation*}
  cannot be verified unless $ \tilde{ v } = \nabla_{  } {h}^{\star}_0,\,\,\mathbb{ P }_0 $a-e.
\end{proof}
We are now ready to make the following claim:
\begin{proposition}[KALE's gradient flow]
  The Wasserstein-2 KALE's gradient flow of KALE on $ \mathcal  P_2(\mathbb{R}^d) $ follows:
  \begin{equation*} 
  \begin{aligned}
    \partial_t \mathbb{ P }_t - \text{div}\left(\mathbb{ P }_t \nabla_{  }  \frac{ \delta \text{KALE} }{ \partial \mathbb{ P } }\right) = 0
  \end{aligned}
  \end{equation*}
\end{proposition}

\begin{proof}
  This is a direct application of \cite[Definition
  11.1.1]{ambrosio2008gradient} using the expression of KALE's strong
  subdifferential of the form $ \left ( i \times v \right )_{\#}\mathbb{ P } $.
\end{proof}

Now that we identified the expression of the KALE gradient flow, we will show
that the KALE gradient flow admits a unique solution. 
To prove that the KALE gradient flow admits a unique solution is to prove that
KALE is $ -M $-semiconvex, for some $ M > 0 $.

\begin{lemma}\label{lemma:appendix-kale-displacement-convexity}
	$ \mathbb{ P }   \longmapsto \text{KALE}(\mathbb{ P } \mid \mid
	\mathbb{ Q }) $ is $ -\frac{ K_{1d} + 4 \sqrt {K K_{2d}}  }{ \lambda }$-geodesically convex.
\end{lemma}
\begin{proof}
  Let $ \mathbb{ P }_a, \mathbb{ P}_b \in \mathcal  P_2( \Rd)
  $, and consider an admissible coupling $ \gamma  \in \Gamma(\mathbb{ P }_a, \mathbb{ P }_b)$ with
  associated transport costs (for various $ p $) $ C_{p}(\gamma) = (\int_{ }^{  } \left \| x - y
  \right \|^{p}\text{d} \gamma(x, y))^{\frac{1}{p}} $. We consider $\left (
    \mathbb{ P }_t \right )_{0 \leq t \leq 1}$ (where $\mathbb{ P }_t = \left ( t \pi^{2} + \left ( 1 - t
  \right ) \pi^{1} \right )_{\#}\gamma $) a constant-speed geodesic between $
\mathbb{ P }_a $ and $ \mathbb{ P }_b $.
    To prove the geodesic convexity of the KALE, we follow a similar approach
    as in \cite{chizat_global_2018} (Lemma B.2). In particular, we show that $ t
    \longmapsto  g(t) = \text{KALE}( \mathbb{ P }_t \mid \mid \mathbb{ Q
    }) $ has an $ MC_{2}^{2}(\gamma) $-Lipschitz derivative, with some $ M $ to be determined.
    Using a similar implicit function theorem argument as in the proof of \cref{lemma:KALE-diff}, we have:
    \begin{equation*} 
    \begin{aligned}
      g'(t) = \int_{  }^{  } (x - y)^{\top} \nabla_{  }  {h}^{\star}_{t}(ty + (1 - t) x) d \gamma(x, y).
    \end{aligned}
    \end{equation*}
    Given $ t_1, t_2 $, we thus have:
    \begin{equation*} 
    \begin{aligned}
      \lvert g'(t_1) - g'(t_2) \rvert  \leq (I) + (II),
    \end{aligned}
    \end{equation*}
    where:
    \begin{equation*} 
    \begin{aligned}
      (I) &= \left \lvert \int_{  }^{  } (x - y)^{\top} \left ( \nabla_{  }  {h}^{\star}_{t_1}( t_1y + \left ( 1 - t_1 \right ) x )
      - \nabla_{  }   {h}^{\star}_{t_1}(t_2 y + \left ( 1 - t_2 \right )x) \right ) d \gamma(x, y) \right \rvert \\
	  &\leq \left \| {h}^{\star}_{t_1} \right \| \sqrt {K_{2d}}(t_2 - t_1)\int_{  }^{  } \left \| x- y \right \|^{2} \text{d} \gamma(x, y)
	  \leq \frac{ 4 \sqrt { K K_{2d}}}{ \lambda } (t_1 - t_2) C_2^{2}(\gamma)
  \end{aligned}
    \end{equation*}
    and:
    \begin{equation*} 
    \begin{aligned}
      (II) &= \int_{  }^{  } (x - y)^{\top} \left ( \nabla_{  }  {h}^{\star}_{t_1}( t_2y + \left ( 1 - t_2 \right ) x )
      - \nabla_{  }   {h}^{\star}_{t_2}(t_2 y + \left ( 1 - t_2 \right )x) \right ) d \gamma(x, y) \\
	&= \int_{  }^{  } \sum\limits_{ i=1 }^{ d } (x_i - y_i)
	\left \langle {h}^{\star}_{t_1} - {h}^{\star}_{t_2}, \frac{ \partial k_{t_2 y +(1 - t_2) x } }{ \partial x_i }\right \rangle  d \gamma(x, y)\\
	&\numrel{\leq}{{e_1}}  \int_{  }^{  } \left \| {h}^{\star}_{t_1} - {h}^{\star}_{t_2} \right \|\sum\limits_{ i=1 }^{ d } \lvert  x_i -y_i  \rvert 
	\left  \| \frac{ \partial k_{t_2 y + \left ( 1 - t_2 \right )x} }{ \partial x_i }  \right \| d \gamma(x, y) \\%& \text{(Cauchy-Schwarz on $ \mathcal  H $)}\\
														   	&\numrel{\leq}{{e_2}}  \sqrt {K_{1d}} \left \| {h}^{\star}_{t_1} - {h}^{\star}_{t_2} \right \|\int_{  }^{  } \sqrt {\left \| x -y \right \|}^2 d \gamma(x, y)\\% & \text{(Cauchy-Schwarz on $ \mathbb{R}^d$)}\\
														    &\numrel{\leq}{{e_3}}  \frac{ (t_2 - t_1)K_{1d}C_2^2(\mathbb{ P }_a, \mathbb{ P}_b) }{ \lambda }  \\%  & \text{(\cref{lemma:kale-sensibility-wrt-P} and Jensen Inqeuality)}
      \end{aligned}
    \end{equation*}
    where (\ref{e_1}) follows from Cauchy-Schwarz on $ \mathcal  H $, (\ref{e_2}) uses Cauchy-Schwarz on $ \mathbb{R}^d$ and (\ref{e_3}) relies on \cref{lemma:kale-sensibility-wrt-P} and Jensen inequality.
     We thus conclude that $ g'(t) $ is  $ M C_2^{2}(\gamma) $-Lipschitz, with
     $ M = \frac{ K_{1d} + 4 \sqrt {K K_{2d}}  }{ \lambda }$,
     and thus that $ \text{KALE} $ is $ -M $-geodesically semiconvex
  \end{proof}
  The geodesic convexity of the KALE allows to conclude the proof of
  \cref{prop:KALE-gf}: indeed, since the KALE is geodesically semiconvex in $ \mathbb{ P } $, and
  admits strong extended Fr\'{e}chet subdifferentials, we conclude that the KALE
  gradient flow solutions exist and are unique, as guaranteed by \cite[Theorem
  11.2.1]{ambrosio2008gradient}. \qed

  \section{Proof of \cref{prop:KALE-flow-cvg}} \label{proof:kale-flow-cvg}
We recall the following definitions: given a positive measure $ \mathbb{ P } $, and a function $ f \in \mathcal  C^1(\Rd) $, the  weighted Sobolev \emph{semi-norm} of $ f $ is given by:
\begin{equation*} 
\begin{aligned}
  \left \| f \right \|_{\dot{H}(\mathbb{ P })} = \left ( \int_{  }^{  } \left \| \nabla_{  } f  \right \|^2 \text{d} \mathbb P \right )^{\frac{1}{2}}.
\end{aligned}
\end{equation*}

Note the important role of the weighted Sobolev semi-norm in the energy dissipation formula of KALE's gradient flow:
\begin{equation} \label{eq:kale-dissipation-sobolev}
\begin{aligned}
  \frac{ \text{d} \text{KALE}(\mathbb{ P }_t \mid \mid \mathbb{ Q }) }{ \text{d}t } = -\int_{  }^{  } (1 + \lambda)^2\left \| \nabla_{  }  {h}^{\star}  \right \|^2 \text{d} \mathbb P_t = - (1 + \lambda)^2\left \| {h}^{\star} \right \|^2_{\dot{H}(\mathbb P_t)}.
\end{aligned}
\end{equation}

By duality, one can define the (possibly infinite) negative \emph{weighted negative
Sobolev distance} \cite{arbel_maximum_2019_1} between $ \mu $ and $ \nu $:
\begin{equation*} 
\begin{aligned}
  \left \| \mu - \nu \right \|_{ \dot{H}^{-1}(\mathbb{ P })} &= \sup_{ \left \|f
  \right \|_{\dot{H}(\mathbb{ P })} \leq 1 } \Bigg \lvert \int_{  }^{  } f
\text{d}(\mu - \nu) \Bigg \rvert .
\end{aligned}
\end{equation*}
As proven in \cite{otto_generalization_2000}, the weighted negative Sobolev
distance linearizes the Wasserstein distance, and one can formally write:
\begin{equation*} 
\begin{aligned}
  W_2(\mu, \mu+ d \mu) = \left \| d \mu \right \|_{ \dot{H}^{-1}(\mathbb{ P })} + o( d \mu).
\end{aligned}
\end{equation*}
Moreover, for all $ f \in \mathcal  C^1(\Rd) $, and $ \mu \in
\mathcal  M(\Rd),$ one has:
\begin{equation} \label{eq:weighted-sobolev-cs}
\begin{aligned}
  \int_{  }^{  } f \text{d} \mu \leq \left \| f \right \|_{\dot{H}(\mathbb P)} \left \| \mu \right \|_{\dot{H}^{-1}(\mathbb P)}.
\end{aligned}
\end{equation}
To prove \cref{prop:KALE-flow-cvg}, we use the $ \lambda $-strong concavity of
$ \mathcal  K(h, \mathbb{ P }) $ w.r.t. $ h $ :
\begin{equation*} 
\begin{aligned}
  \text{KALE}(\mathbb{ P }  \mid \mid \mathbb{ Q }) = (1 + \lambda) \mathcal  K( {h}^{\star}, \mathbb{ P })
  &\leq (1 + \lambda)(\mathcal  K(0, \mathbb{ P }) + \left \langle {h}^{\star},
  \nabla_{ h } \mathcal  K (0, \mathbb{ P })  \right \rangle -
  \frac{\lambda}{2} \left \| {h}^{\star} \right \|^2)\\
  & \leq  (1 + \lambda)\left \langle {h}^{\star}, \mu_{\mathbb{ P }} - 
  \mu_{\mathbb{ Q }}\right \rangle = (1 + \lambda)\int_{  }^{  } {h}^{\star}(x) \text{d} \mathbb P - \int_{  }^{  } {h}^{\star}(x) \text{d} \mathbb Q \\
  & \leq (1 + \lambda)\left \| h \right \|_{\dot{H}(\mathbb{ P })}
			    \left \| \mathbb{ P } - \mathbb{ Q } \right
			    \|_{\dot{H}^{-1}(\mathbb{ P })} \leq (1 + \lambda) C \left \| h \right \|_{\dot{H}(\mathbb{ P })}. \\
\end{aligned}
\end{equation*}
Here we  successively applied \cref{eq:weighted-sobolev-cs} and the hypothesis $
\left \| \mathbb{ P } - \mathbb{ Q } \right \|_{\dot{H}(\mathbb{ P })} \leq C $. Recalling \cref{eq:kale-dissipation-sobolev}, one has:
\begin{equation*} 
\begin{aligned}
  \frac{ d \text{KALE}(\mathbb{ P }_t, \mathbb{ Q }) }{ \text{d}t } &\leq  - \frac{ \text{KALE}(\mathbb{ P }_t \mid \mid \mathbb{ Q })^2 }{ C^2 } \Longrightarrow
  \frac{ d (1 / \text{KALE}(\mathbb{ P }_t, \mathbb{ Q })) }{ \text{d}t } \geq \frac{1}{C},
\end{aligned}
\end{equation*}
from which the desired inequality follows.\qed

\paragraph{Proof of \cref{prop:KALE-noise-injection}} 
We rely on the proof technique used in \cite[E.1]{arbel_maximum_2019_1}.
From \cref{lemma:differential-kernel}, we get that assumptions A, D of \cite{arbel_maximum_2019_1} hold with $ L = \sqrt {K K_{2d}}  $ and $ \lambda^2 = K_{2d} $.
Moreover, we know from
\cref{lemma:bounded-kale-witness-function} that $ {h}^{\star} $ is $
\frac{4K}{\lambda} $-Lipschitz.  From these smoothness conditions, all steps in
{\cite[E.1]{arbel_maximum_2019_1}}, follow until:
\begin{equation*} 
\begin{aligned}
  \text{KALE}(\mathbb{ P }_{n+1} \mid \mid \mathbb{ Q }) - \text{KALE}(\mathbb{ P }_{n} \mid \mid \mathbb{ Q }) \leq - \gamma \left ( 1 - \frac{3}{2} \gamma \sqrt {K K_{2d}}  \right ) \mathcal  D_{\beta_n}(\mathbb{ P }_n) + \gamma \sqrt {K_{2d}} \beta_n \left \| {h}^{\star} \right \| \mathcal  D_{\beta_n}(\mathbb{ P }_n)^{\frac{1}{2}}.
\end{aligned}
\end{equation*}
Now, given that $ \left \| {h}^{\star} \right \|^2 \leq  \frac{ 2
\text{KALE}(\mathbb{ P }_n, \mathbb{ Q }) }{ \lambda } $ and that 
$ \frac{8 K_{2d}\beta_n^2}{\lambda^2} 
\text{KALE}(\mathbb{ P }_n, \mathbb{ Q })
\leq \mathcal  D_{\beta_n}(\mathbb{ P }_n) $
we have:
\begin{equation*} 
\begin{aligned}
  \text{KALE}(\mathbb{ P }_{n+1} \mid \mid \mathbb{ Q }) - \text{KALE}(\mathbb{ P }_{n} \mid \mid \mathbb{ Q })
  &\leq - \gamma \left ( 1 - \frac{3}{2} \gamma \sqrt {K K_{2d}}  \right ) \mathcal  D_{\beta_n}(\mathbb{ P }_n) + \gamma \sqrt {\frac{2}{8}} D_{\beta_n}(\mathbb{ P }_n)  \\
  &\leq -\frac{\gamma}{2} \left ( 1 - 3 \gamma \sqrt {K K_{2d}}  \right ) \mathcal  D_{\beta_n}(\mathbb{ P }_n) \\
  &\numrel{\leq}{{e_4}}  -  4 \gamma \left ( 1 - 3 \gamma \sqrt {K K_{2d}}  \right ) \frac{ K_{2d}}{
  \lambda^2 } \beta_n^2\text{KALE}(\mathbb{ P } \mid \mid \mathbb{ Q })\\
  &\numrel{\leq}{{e_5}} - \Gamma \beta_n^2 \text{KALE}(\mathbb{ P }_n \mid \mid \mathbb{ Q }),
\end{aligned}
\end{equation*}
where (\ref{e_4}) uses the noise schedule assumption and in (\ref{e_5}) we noted   
$ \Gamma =  4 \gamma \left ( 1 - 3 \gamma \sqrt {K K_{2d}}  \right ) \frac{ K_{2d}
}{ \lambda^2 } $, and the result follows as in \cite{arbel_maximum_2019_1}.

\section{Proof of \cref{prop:kale-descent-vs-discrete-kale-flow}}\label{proof:kale-descent-vs-kale-flow}

We recall the update equations defining the trajectories $ (Y^{(i)}_n)_{n \leq n_{\max_{  }}} $ and $ (\bar{ Y }^{(i)}_n)_{n \leq n_{\max_{  }}} $:
 \begin{equation} \label{eq:appendix-kale-particle-descent}
 \begin{aligned}
   Y^{(i)}_{n+1} = Y^{(i)}_n -\gamma (1 + \lambda)\nabla
   \widehat{h}^\star_{n}( Y^{(i)}_n), \\
   \bar{ Y }_{n+1}^{(i)} = \bar{ Y }^{(i)}_n-\gamma (1 + \lambda)\nabla
   {h}^{\star}_{n}( \bar{ Y }^{(i)}_{n}) .
 \end{aligned}
 \end{equation}
 We denote $ c_n = \sqrt { \frac{1}{N} \sum\limits_{ i=1 }^{ N } \mathbb{E} \left \|
 \bar{ Y }^{(i)}_n -  Y^{(i)}_n \right \|^2} $. Note that
 $$ \mathbb{E} W_2( \overline{ \mathbb{ P } }_{n}^{N}, \widehat{ \mathbb{ P } }_{n}^{N})^2 \leq
 \frac{1}{N} \sum\limits_{ i=1 }^{ N } \mathbb{E}_{  }\left [ \left \|
  Y^{(i)}_{n+1} - \bar{ Y }^{(i)}_{n+1} \right \|^2 \right ] = c_{n}^2.$$
 The iterates $ c_{n} $ satisfy the following recursion:

\begin{equation*} 
	\begin{aligned}
	  c_{n+1} &= \sqrt{\frac{1}{N}
			\sum\limits_{ i=1 }^{ N } \mathbb{E}_{  }\left [ \left \| { Y }^{(i)}_{n+1} - \bar{ Y
				}^{(i)}_{n+1} \right \|^2 \right ] }                                                                                                                             \\
		                                                                              & \leq \sqrt{\frac{1}{N}
		\sum\limits_{ i=1 }^{ N } \mathbb{E}_{  }\left [ \left \| { Y }^{(i)}_{n} - \bar{ Y
		    }^{(i)}_{n}   - \gamma (1 + \lambda) \left (  \nabla_{  } \widehat{ h }^{\star}_{n}(
		{ Y }^{(i)}_n) - \nabla_{  } {h}^{\star}_{n}( \bar{ Y }^{(i)}_n)   \right ) \right \|^2\right ]}                                                                                 \\
		                                                                              & \leq c_n +
											      \underbrace{ \frac{\gamma(1 + \lambda)}{\sqrt {N} }\sqrt{ \sum\limits_{ i=1 }^{ N } \mathbb{E}_{  }\left [ \left \|   \nabla_{  } \widehat{ h }^{\star}_{n}(
												{ Y}^{(i)}_n) - \nabla_{  } h^{\star}_{n}( \bar{ Y }^{(i)}_n)   \right \|^2\right ]}
	      }\limits_{\overset{\Delta}{=}A}.
	\end{aligned}
\end{equation*}
Using a triangular inequality, we now split (A) into terms that will be handled differently:
\begin{equation*} 
	\begin{aligned}
		c_{n+1} & \leq c_n +
		\gamma (1 + \lambda) \left (
		\underbrace{ \frac{1}{\sqrt {N} }\sqrt{ \sum\limits_{ i=1 }^{ N } \mathbb{E}_{  }\left [
		      \left \|   \nabla_{  } \widehat{ h }^{\star}_{n}({ Y }^{(i)}_n) - \nabla_{  }
		      \widehat{ h }^{\star}_{n}( \bar{
			Y }^{(i)}_n)   \right \|^2\right ]} }\limits_{(i)} \right .\\
			+& \left .
		\underbrace{ \frac{1}{\sqrt { N} }\sqrt{ \sum\limits_{ i=1 }^{ N } \mathbb{E}_{  }\left [
		      \left \|   \nabla_{  } \widehat{ h }^\star_{n}( \bar{ Y
	      }^{(i)}_n) - \nabla_{  } \bar{ h }^\star_{n}( \bar{ Y }^{(i)}_n)   \right
      \|^2\right ]} }\limits_{(ii)}
      +\underbrace{ \frac{1}{\sqrt { N} }\sqrt{ \sum\limits_{ i=1 }^{ N } \mathbb{E}_{  }\left [
		\left \|   \nabla_{  } \bar{ h }^\star_{n}( \bar{ Y }^{(i)}_n) - \nabla_{  }
      h_{n}^\star( \bar{
    Y }^{(i)}_n)   \right \|^2\right ]} }\limits_{(i i i )} \right )
\end{aligned}
\end{equation*}

Where we introduced the notation $ \bar{ h }^\star_n = \arg \max_{ h} \mathcal
K(h, \overline{ \mathbb{ P } }_n^{N}) $, the witness function that estimates
the \emph{true} witness function $ {h}^{\star}_n $ using $ \overline{ \mathbb{
P } }_n^{N} $, the empirical version of $ \mathbb{ P }_n $, instead of $
\mathbb{ P }_n $. Let us explain the source of each of the terms in the last inequality: 
\begin{itemize}
  \item (i) comes from evaluating the velocity field $ \widehat{ h
    }^\star_n $ at different points ${ Y }^{(i)}_n $ and $ \bar{ Y
  }^{(i)}_n $,
  \item (ii) comes from using \emph{biased} samples $ \{ Y^{(i)}_n \}_{i=1}^{N}
    $ to compute $ \widehat{ h }^\star_{n} $, and unbiased samples $ \{ \bar{ Y
    }^{(i)}_n\}_{i=1}^{N} $ to compute $    \bar{ h }^\star_{n} $.
  \item (iii) comes from the use of a finite number of unbiased samples to
    compute $ \bar{ h }^\star_{n} $.
\end{itemize}
After controlling (i), (ii), (iii), as detailed below, we get the following upper
bound:

\begin{equation*} 
\begin{aligned}
  c_{n+1} \leq c_n \gamma ( 1 + \lambda) \left ( 1 + \frac{ 4 \sqrt {KK_{2d}} + K_{2d}}{ \lambda }
  \right ) + \frac{ \gamma ( 1 + \lambda) }{ \lambda } \sqrt { \frac{KK_{2d}(1 + e^{\frac{8K}{\lambda}})}{N}} .
\end{aligned}
\end{equation*}
We use \cite[Lemma 26]{arbel_maximum_2019_1} to conclude:
$$ c_n = \sqrt {{ \frac{ 2KK_{1d}(1 + e^{\frac{8K}{\lambda}}) }{ N }}} \times
\frac{1}{4 \sqrt {KK_{1d}}+ K_{2d}}
(e^{ \gamma (1 + \lambda)\frac{ 4 \sqrt {K K_{1d}} + K_{2d}   }{ \lambda }n } - 1). $$
The result on $ \mathbb{E} W_2( \bar{ \mathbb{ P } }_n, \widehat{ \mathbb{ P } }_n) $
follows by noting that $ \mathbb{ E } W_2( \bar{ \mathbb{ P } }_n^{N},
\widehat{ \mathbb{ P } }_n^{N}) \leq \sqrt {\mathbb{E} W_2^{2}( \bar{
\mathbb{ P } }_n^{N}, \widehat{ \mathbb{ P } }_n^{N})}  $ by Jensen's
inequality. \qed

\subsection{Control of the 3 error terms}
\paragraph{Controlling (i)}

To control the first term, we rely on the RKHS derivative reproducing property
\cite{zhou_derivative_2008}: $ \frac{ \partial h }{ \partial x_i } = \left
\langle \partial_i k_x, h \right \rangle  $,
\cref{assump:smooth-kernel},
and on the uniform bound on $ \left \| {h}^{\star}
\right \|$ (for all $ \mathbb{ P } $, $ \mathbb{ Q } $) given by
(\cref{lemma:bounded-kale-witness-function}) :
\begin{equation*} 
\begin{aligned}
  \left \|\nabla_{  } \widehat{ h }^\star_{n} ( { Y }_n^{(i)}) - \nabla_{  } \widehat{ 
  h}^\star_{n} ( \bar{ Y }_n^{(i)})  \right \|^2 \leq \sum\limits_{ i=1 }^{ d } \left \| 
  \partial_i k_{{ Y }_n^{(i)}} - 
  \partial_i k_{\bar{ Y }_n^{(i)}}\right \|^2 \left \| \hat{ h
}_{n} \right \|^2 = \frac{ 16 KK_{2d}}{ \lambda^2 } \left \| { Y }_n^{(i)} - \bar{ Y
}_i^{(n)} \right \|^2.
\end{aligned}
\end{equation*}
Consequently, we have
\begin{equation*} 
\begin{aligned}
  (i) = \frac{1}{\sqrt {N} } \sqrt {\sum\limits_{ i=1 }^{ N } \mathbb{E}\left \| \nabla_{  }  \widehat{ h }_n^{\star}( { Y }^{(i)}_n) - \nabla_{  }  \widehat{ h }_n^\star( \bar{ Y }^{(i)}_n)   \right \|^2} \leq \frac{ 4 \sqrt {KK_{2d}}  }{ \lambda \sqrt {N}  } c_n.
\end{aligned}
\end{equation*}

\paragraph{Controlling (ii)} 
To control $ (ii) $, we rely on \cref{lemma:kale-sensibility-wrt-P}, that
guarantees that  $ \text{KALE}(\mathbb{ P } \mid \mid \mathbb{ Q }) $ is
$ \frac{ \sqrt {K_{1d}}  }{\lambda} $-Lipschitz in $ \mathbb{ P } $ and $ \mathbb{ Q } $, when $
\mathcal  P(\Rd) $ is endowed with the Wasserstein-2 metric:

  \begin{equation*} 
  \begin{aligned}
  \left \|\nabla_{  } \widehat{ h }^\star_{n} ( \bar{ Y }_n^{(i)}) - \nabla_{  } \bar{
  h}^\star_n ( \bar{ Y }_n^{(i)})  \right \|^2
  &= \sum\limits_{ j=1 }^{ d } \left ( \partial_j \widehat{ h }^\star_n( \bar{ Y }^{(i)}_n)  - \partial_j \bar{ h }^\star_n( \bar{ Y }^{(i)}_n)  \right )^2 \\
  & \leq K_{1d} \left \| \hat{ h }^\star_{n} - \bar{ h }^\star_{n} \right \|^2.
  \end{aligned}
  \end{equation*}
  Consequently, using \cref{lemma:kale-sensibility-wrt-P}, we have:
  \begin{equation*} 
  \begin{aligned}
\left \|\nabla_{  } \widehat{ h }^\star_{n} ( \bar{ Y }_n^{(i)}) - \nabla_{  } \bar{
h}^\star_n ( \bar{ Y }_n^{(i)})  \right \|^2 &\leq \frac{ K_{1d}^2 }{ \lambda^2 } W_2( \widehat{ \mathbb{ P } }^{N}_n, \bar{ \mathbb{ P } }^{N}_n)^2 \\
\Longrightarrow (i i ) = \frac{1}{\sqrt {N} } \sqrt {\sum\limits_{ i=1 }^{ N } \mathbb{E}\left \| \nabla_{  }  \widehat{ h }^{\star}( \bar{ Y }^{(i)}_n) - \nabla_{  }  { h }^\star( \bar{ Y }^{(i)}_n)   \right \|^2} &\leq \frac{ K_{1d}  }{ \lambda \sqrt {N}  } \sqrt { \mathbb{E} W_2^{2}( \widehat{ \mathbb{ P } }^{N}_n, \bar{ \mathbb{ P } }^{N}_n)} \leq \frac{ K_{1d} }{ \lambda \sqrt {N} } c_n.
  \end{aligned}
  \end{equation*}

  \paragraph{Controlling (iii)} 
  In  (iii), the witness function $ \bar{ h }^\star_n $ is an empirical version of $
  {h}^{\star}_n $. Repeating the first lines of (ii), we have:
  \begin{equation*} 
  \begin{aligned}
  \left \|\nabla_{  } \bar{  h }^\star_{n} ( \bar{ x }_n^{(i)}) - \nabla_{  } 
  h^\star_n ( \bar{ x }_n^{(i)})  \right \|^2
  & \leq K_{1d}\left \| \bar{ h }^\star_n- {h}^{\star}_n\right \|^2.
  \end{aligned}
  \end{equation*}
  We could use the bound given in $ (i i) $ to get a bound on  $ \left \|\bar{
  h }^\star_n - {h}^{\star}_n \right \| $, but the sample complexity of the
  Wasserstein distances scales in $ \mathcal  O(n^{-1 / d}) $, which is much slower than
  our target rate $ 1/{\sqrt {N}}  $  \citep{weed_sharp_2017}. Instead, we rely
  on the concentration inequality given by \cref{expectaiton-empiricial-h},
  ensuring that $ \mathbb{E}\left \| \bar{ h }^\star_n - {h}^{\star}_n \right
  \|^2 \leq \frac{ 2K(1 + e^{ \frac{8K}{\lambda}}) }{ N \lambda^2 }  $.
  Following this, we have:
  \begin{equation*} 
  \begin{aligned}
    (i i i) = \frac{1}{\sqrt {N} } \sqrt {\sum\limits_{ i=1 }^{ N } \mathbb{E}\left \| \nabla_{  }  \bar{ h }^{\star}_n( \bar{ x }^{(i)}_n) - \nabla_{  }  { h }^\star( \bar{ x }^{(i)}_n)   \right \|^2}
  \leq \frac{1}{\lambda} \sqrt { \frac{2K K_{1d}(1 + e^{\frac{8K}{\lambda}})}{ N}} .
  \end{aligned}
  \end{equation*}

\section{Auxiliary Lemmas}\label{app-sec:auxiliary-lemmas}

\begin{lemma}[Uniform smoothness of the KALE witness function]\label{lemma:bounded-kale-witness-function}
  Under \cref{assump:bounded-kernel}, and for all $ \mathbb{ P } $, $ \mathbb{
  Q }  $, the following inequalities hold:
  \begin{equation*} 
  \begin{aligned}
    \frac{\lambda}{2} \left \| {h}^{\star} \right \|^2\leq
    \text{KALE}(\mathbb{ P } \mid \mid \mathbb{ Q }) \leq  2 \sqrt {K}  \left \|
    {h}^{\star} \right \|,
  \end{aligned}
  \end{equation*}
implying $ \left \|{h}^{\star}  \right \| \leq \frac{ 4 \sqrt {K}   }{
  \lambda } $. We also have the finer estimate $ \left \| {h}^{\star} \right \|
  \leq \frac{ 2 \text{MMD}(\mathbb{ P } \mid \mid \mathbb{ Q }) }{ \lambda } $.

\end{lemma}
\begin{proof}
  The right inequality follows from the proof of \cref{prop:KALE-flow-cvg}. Indeed, we have:
  \begin{equation*} 
  \begin{aligned}
    \text{KALE}(\mathbb{ P } \mid \mid \mathbb{ Q })
    &\leq \left \langle {h}^{\star}, \mu_{\mathbb{ P }} - \mu_{\mathbb{ Q }} \right \rangle
    \leq \left \| {h}^{\star} \right \|  ( \left \| \mu_{\mathbb{ P }} \right
    \| + \left \| \mu_{\mathbb{ Q }} \right \|)
     \leq 2 \sqrt {K}  \left \| {h}^{\star} \right \|.
  \end{aligned}
  \end{equation*}
  The left inequality can be noticed using KALE's \emph{dual formulation} \cref{eq:KALE-primal}

  \begin{equation}
  \begin{aligned}
    \text{KALE}(\mathbb{ P } \mid \mid \mathbb{ Q }) &= \underbrace{ \int_{
      }^{ } \left ( {f}^{\star} (\log {f}^{\star} - 1) + 1 \right )  \text{d}
      \mathbb Q }\limits_{\geq  0} + \frac{1}{2
  \lambda} \left \| \int_{  }^{  } {f}^{\star}(x) \text{d} \mathbb Q(x) - \mu_{\mathbb{ P
  }} \right \|^2 \\
  &\geq 
\frac{1}{2
  \lambda} \left \| \int_{  }^{  } {f}^{\star}(x) \text{d} \mathbb Q(x) - \mu_{\mathbb{ P
  }} \right \|^2 = \frac{\lambda}{2} \left \| {h}^{\star} \right \|^2.
  \end{aligned}
  \end{equation}
  To get the finer estimate, we keep track of $ \frac{\lambda}{2} \left \| {h}^{\star} \right \|^2_{\mathcal H } $ term. By convexity of $ \exp $, we have:
	\begin{equation*} 
		\begin{aligned}
			\underbrace{ 1 + \int_{  }^{  } h \text{d} \mathbb P - \int_{  }^{  } e^{h} \text{d} \mathbb Q -
			\frac{\lambda}{2} \left \| h \right \|^2 }\limits_{\mathcal  K(h, \mathbb{ P })} & \leq \int_{  }^{  } h \text{d} \mathbb P -
			\int_{  }^{  } h \text{d} \mathbb Q - \frac{\lambda}{2} \left \| h \right \|^2.
		\end{aligned}
	\end{equation*}
	Recalling now that $ \mathcal  K( {h}^{\star}, \mathbb{ P }) \geq  \mathcal  K(0,
		\mathbb{ P }) = 0 $, we must have:
	\begin{equation*} 
		\begin{aligned}
			\int_{  }^{  } {h}^{\star} \text{d} \mathbb P -
			\int_{  }^{  } {h}^{\star} \text{d} \mathbb Q - \frac{\lambda}{2} \left \| {h}^{\star} \right \|^2 \geq
			0 \\
			\Longrightarrow \left \|{h}^{\star}  \right \| \leq \frac{ 2 \Vert f_{\mathbb{ P }, \mathbb{ Q }}\Vert }{ \lambda }
		\end{aligned}
	\end{equation*}
	Where the last line used the Cauchy-Schwarz inequality.
\end{proof}

\begin{lemma}\label{expectaiton-empiricial-h}
  Under \cref{assump:bounded-kernel}, and using the notations of \cref{proof:kale-descent-vs-kale-flow}, we have:
  \begin{equation*} 
  \begin{aligned}
    \mathbb{E} \left \| \bar{ h }^\star_n - {h}^{\star}_n \right \|^2 &\leq
    \frac{2K(1 + e^{\frac{8K}{\lambda}})}{ N\lambda^2 }
  \end{aligned}
  \end{equation*}
\end{lemma}  
\begin{proof}
  We first notice, as explained in \cite{arbel_generalized_2020_1} (Proposition 12), that
  $ \left \| \bar{ h }^\star_n - {h}^{\star}_n\right \| \leq
  \frac{1}{\lambda} \left \| \nabla_{  }   \widehat{ \mathcal  L }(
  {h}^{\star}_n) - \nabla_{  }   \mathcal  L( {h}^{\star}_{n})\right \| $
  where $ \mathcal  L  = 1 + \int_{  }^{  } h \text{d} \mathbb P - \int_{  }^{
  } e^{h} \text{d} \mathbb Q$ is the KL objective, and $ \widehat{ \mathcal  L }(h)
  = \int_{  }^{  } h \text{d} \bar{\mathbb{ P }}_n^{N} - \int_{  }^{  } e^{h} \text{d}
\widehat{ \mathbb{ Q }}^{N} + 1 $ is its empirical equivalent.
  We then use \cite{JMLR:v18:17-032} (Proposition A.1, notice that their
  statement also holds for $ ({\mathbb{E}_{  } \| \int_{  }^{  } r
  \text{d} \mathbb P_n - \int_{  }^{  } r \text{d} \mathbb P \|^2})^{1 / 2}  $), to
  get:
  \begin{equation*} 
  \begin{aligned}
  E \left \| \nabla_{  }   \widehat{ \mathcal  L }(
  {h}^{\star}_n) - \nabla_{  }   \mathcal  L( {h}^{\star}_n)\right \|^2
  \leq & \mathbb{ E } \left \| \int_{  }^{  }k(x, \cdot)\text{d} \mathbb P_n - \int_{  }^{
} k(x, \cdot) d \bar{\mathbb{ P } }^N_n\right \|^2 \\
  &+ \mathbb{ E } \left \| \int_{  }^{  } k(x, \cdot) e^{ {h}_n^{\star}} \text{d} \mathbb Q - \int_{  }^{  } k(x, \cdot)e^{ {{h}}^{\star}_n} \text{d} \widehat{ \mathbb Q}^{N}\right \|^2 \\
  &\leq  \frac{  K(1 + e^{\frac{8K}{\lambda}})  }{ N  },
  \end{aligned}
  \end{equation*}
  where we used the Cauchy-Schwarz inequality on $ \mathcal  H $ and
  \cref{lemma:bounded-kale-witness-function} to bound the squared norm of
  $ x  \longmapsto k(x, \cdot) e^{ {h}^{\star}(x)} $.
\end{proof}

\begin{lemma}\label{lemma:differential-kernel} Under \cref{assump:smooth-kernel}, The maps $ x  \longmapsto k_x (\overset{\Delta}{=} k(x, \cdot)) $ and $ x  \longmapsto \nabla_{  } k_x  $ are differentiable. Moreover, we have
  \begin{equation*}
  \begin{aligned}
    \left \| k_x - k_y \right \| \leq \sqrt {K_{1d}}  \left \| x - y \right \| \\
    \left \| \nabla_{  }  k_x - \nabla_{  }  k_y \right \| \leq \sqrt {K_{2d}}  \left \| x - y \right \| \\
  \end{aligned}
  \end{equation*}
\end{lemma}
\begin{proof}
  We prove the differentiability and the Lispchitzness property for the map $ x  \longmapsto k_x $; the arguments can be straightforwardly adapted to the case of $ x  \longmapsto \nabla_{  } k_x  $. To prove the differentiability, we build upon \cite[Lemma 4.34]{steinwart2008support}, that guarantees that $ x  \longmapsto k(x, \cdot) $ admits partial derivatives for all $ i $, noted $ \partial_i \phi(x) $.  We finish the proof by construction: let $ D \phi(x): \mathbb{R}^d  \longmapsto \mathcal  H $ our candidate differential, defined as $ D \phi(x)(\Delta) = \sum_{ i=1 }^{ d } \Delta_i \partial_i \phi(x) $ for all $ \Delta \in  \mathbb{R}^d $. We show that $ D \phi(x) $ is the differential of $ \phi $ at x using a simple telescopic argument:
    let us note $ \left ( x + \Delta \right )_{:i} = (x_1 + \Delta_1, \dots, x_{i} + \Delta_{i}, x_{i + 1}, \dots, x_d) $ for any $ i \in \left \{ 0, \dots, d \right \}  $ with $ (x + \Delta)_{:0} = x $ by convention.  Then:
    \begin{equation*} 
    \begin{aligned}
        \phi(x + \Delta) - \phi(x) &= \sum\limits_{ i=d }^{ 1 } \phi(x+\Delta)_{:i} - \phi((x + \Delta)_{:i-1}) 
    \end{aligned}
    \end{equation*}
    Knowing that $ \phi((x + \Delta)_{:i}) - \phi((x + \Delta)_{:(i-1)}) = \partial_i \phi({(x + \Delta)_{:i-1}}) \Delta_i + o(\lvert  \Delta_i \rvert) $, we have:
   \begin{equation*} 
   \begin{aligned}
       \phi(x + \Delta) - \phi(x)  - D \phi(x)(\Delta) &= \sum\limits_{ i=1 }^{ d } \left ( \partial_i \phi((x + \Delta)_{:i-1}) - \partial_i \phi(x) \right )\Delta_{i} + o(\lvert \Delta_i \rvert) \\
       \Longrightarrow \left \| \phi(x + \Delta) - \phi(x)  - D \phi(x)(\Delta)\right \| &\leq \sum\limits_{ i=1 }^{ d } \lvert \Delta_i \rvert  \left \| \partial_i \phi((x + \Delta)_{:i-1}) - \partial_i \phi(x) \right \| + o(\left \|\Delta \right \|_1)
   \end{aligned}
   \end{equation*}
   From \cite[Lemma 4.34]{steinwart2008support}, we have: that:
   \begin{equation*} 
   \begin{aligned}
   \left \| \partial_i \phi((x + \Delta)_{:i-1}) - \partial_i \phi(x) \right \|^2 &= A - B   \end{aligned}
   \end{equation*}
   where 
   \begin{equation*} 
   \begin{aligned}
       A &= \partial_i \partial_{i + d} k((x + \Delta)_{:i-1}, (x + \Delta)_{:i-1}) - \partial_i \partial_{i+d} k((x + \Delta)_{:i-1}, x))\\
       B &= \partial_i \partial_{i + d} k((x + \Delta)_{:i-1}, x) - \partial_i \partial_{i+d} k(x, x)
   \end{aligned}
   \end{equation*}
   Since $ \partial_i \partial_{i+d} k(x, x') $ is continuous, both $ A $ and  $ B $ tend to $ 0 $ as  $ \left \| \Delta \right \| $ tends to 0.  Thus, we have:
   \begin{equation*} 
   \begin{aligned}
       \left \| \phi(x + \Delta) - \phi(x)  - D \phi(x)(\Delta)\right \| &\leq \sum\limits_{ i=1 }^{ d } o(\lvert \Delta_i \rvert) + o(\left \|\Delta \right \|_1) = o( \left \|\Delta \right \|_{2})
   \end{aligned}
   \end{equation*}
   by equivalency of $  \left \| \cdot \right \|_{1} $ and $ \left \|\cdot \right \|_{2} $ in $ \mathbb{R}^d $.

   Lipschitzness is guaranteed by bounding the operator norm of $ D \phi(x) $:
   \begin{equation*} 
   \begin{aligned}
     D \phi(x) = \sup_{ \left \| \Delta \right \| = 1 } \left \|D \phi(x)\Delta  \right \| \leq \sum\limits_{ i=1 }^{ n } \lvert \Delta_i \rvert \left \| \partial_i \phi(x) \right \| \leq \sqrt { \left \| \Delta \right \|^2} \sqrt {\sum\limits_{ i=1 }^{ n } \left \| \partial_i \phi(x) \right \|^2}  = \sqrt {K_{1d}} 
   \end{aligned}
   \end{equation*}
\end{proof}

\begin{lemma}\label{lemma:kale-sensibility-wrt-P}
For any $ \mathbb{ P }_0 $, $ \mathbb{ P }_1 \in \mathcal  P_2(\Rd) $,
with associated KALE witness functions $ {h}^{\star}_0,   {h}^{\star}_1 $, we
have: \begin{equation*} 
\begin{aligned}
  \left \| {h}^{\star}_1 - {h}^{\star}_0 \right \|^2 \leq \frac{ K_{1d}  }{
  \lambda^2 } W_2(\mathbb{ P }_0, \mathbb{ P }_1)^2.
\end{aligned}
\end{equation*}
\end{lemma}
\begin{proof}
  The optimal functions $ {h}^{\star}_0 $ and $ {h}^{\star}_1 $ are
  characterized by the following optimality condition:
  \begin{equation*} 
  \begin{aligned}
    \int_{  }^{  } k(x, \cdot) \text{d} \mathbb P - \int_{  }^{  } k(x, \cdot) e^{
    {h}^{\star}} \text{d} \mathbb Q - \lambda {h}^{\star} = 0.
  \end{aligned}
  \end{equation*}
Let us now pose $ \text{d}\mathbb{ P }_t = \text{d} \mathbb P_0 + t d\chi $ with $
\text{d}\chi = \text{d} \mathbb P_1 - \text{d} \mathbb P_0 $, and its
associated witness function $ {h}^{\star}_t $. Using an implicit function
theorem argument \cite{lang_fundamentals_1999} between $ t $ and  $ {h}^{\star}_t $, we can write,
with notations of \cref{proof:KALE-diff}:
\begin{equation*} 
\begin{aligned}
  \frac{ d {h}^{\star}_t }{ d t } = \left ( \boldsymbol{L}(
  {h}^{\star}_t) + \lambda I \right )^{-1} \int_{   }^{  }k(x, \cdot) (\text{d} \mathbb P_1
  - \text{d} \mathbb P_0).
\end{aligned}
\end{equation*}
The operator $ \boldsymbol{L} $ is the covariance operator of the measure  $ \tilde{\mathbb{
Q }} = e^{ {h}^{\star}} \mathbb{ Q } $. This operator is compact given that $ k $ is
bounded by \cref{assump:bounded-kernel}. Using the spectral theorem on Hilbert spaces, 
we know that there exists a complete orthonormal system of eigenvectors of $ \boldsymbol{L}
$, with associated eigenvalues $ \left \{ \mu_{i, t} \right \}_{i \in
  \mathbb{N}}  $ for any $ t $. The operator $ \left ( \boldsymbol{L} + \lambda I \right
)^{-1} $ admits an identical eigendecomposition, with eigenvalues $ \left \{
\frac{1}{\lambda + \mu_{i, t}} \right \}_{i \in \mathbb{N}} $: thus, the
operator norm of $ (\boldsymbol{L} + \lambda I)^{-1} $ is upper-bounded by $ 1
/ \lambda $. We can thus extract a bound on $ \left \| {h}^{\star}_1 - {h}^{\star}_0
\right \|^2 $:
\begin{equation*} 
\begin{aligned}
  \left \| {h}^{\star}_1  - {h}^{\star}_0 \right \|^2 &= \left \| \int_{0}^{1} \left ( \boldsymbol{L}( {h}^{\star}_t) + \lambda
  I \right )^{-1} \left ( \int_{  }^{  }k(x, \cdot) \left ( \text{d} \mathbb 
P_1 - \text{d} \mathbb P_0 \right ) \right ) \text{d}t \right \|^2 \\
  & \leq \int_{ 0 }^{ 1 } \left \| \left ( L( {h}^{\star}_t) + \lambda
  I \right ) \int_{  }^{  } k(x, \cdot) \left ( \text{d} \mathbb P_1 - \text{d} \mathbb
P_0\right ) \right \|^2 \text{d}t \quad \\
				& \leq \int_{ 0 }^{ 1 } \frac{1}{\lambda^2} \left \| \int_{
				}^{  } k(x, \cdot) \left ( \text{d} \mathbb P_1 -
			      \text{d} \mathbb P_0\right )  \right \|^2\text{d}t = \frac{1}{\lambda^2}\left \| \int_{
				}^{  } k(x, \cdot) \left ( \text{d} \mathbb P_1 -
			      \text{d} \mathbb P_0\right )  \right \|^2.
\end{aligned}
\end{equation*}
Now, let $ \nu \in \Gamma(\mathbb{ P }_1, \mathbb{ P }_0) $.  Then one has:
\begin{equation*} 
\begin{aligned}
  \int_{  }^{  } k(x, \cdot) \left ( \text{d} \mathbb P_0 - \text{d} \mathbb P_1 \right
    ) &= \int_{  }^{  } \left ( k(x, \cdot) - k(y, \cdot) \right ) d \nu(x, y)\\
  \left \|\int_{  }^{  } k(x, \cdot) \left ( \text{d} \mathbb P_0 - \text{d} \mathbb P_1 
    \right )  \right \|^2  &\leq \int_{  }^{  } \left \|k(x, \cdot) - k(y, \cdot)  \right \|^2
    \text{d} \nu(x, y)\\
		 & \leq K_{1d}\int_{  }^{  } \left \| x - y \right \|^2 \text{d}\nu(x, y)
		 = K_{1d} W_2(\mathbb{ P }_0, \mathbb{ P }_1)^2
\end{aligned}
\end{equation*}
Where we applied first Jensen's inequality and \cref{lemma:differential-kernel}.
\end{proof}

\section{Details of Numerical Experiments and Impact of Noise Injection}\label{app-sec:numerical-details}

\vspace{-1em}

In this section, we provide further details on the experiments in the main
paper. The step size $ \gamma $
used for the KALE particle descent algorithm scales with $ \lambda $ as $
\min_{  }(0.1, \frac{\lambda}{10}) $. For all experiments, we used a Gaussian
kernel $ k(x, y) = \exp(- \frac{ \left \| x - y \right \|^2 }{ 2 \sigma^2 }) $.
The kernel width $ \sigma $ is described for each experiment set.
\paragraph{``Three rings'' experiments}
For this experiment, the number of particles in each distribution was $ N = 300
$, and we used the Newton algorithm to compute the KALE.  We used a kernel width $
\sigma = 0.3 $. We show in \cref{subfig:noise-injection-ring} the impact of
noise injection with a constant noise schedule of $ \beta_n = 0.3 $.

\begin{figure*}[h]
\centering
\begin{subfigure}[t]{0.49\textwidth}
    \centering
    \includegraphics[width=\textwidth]{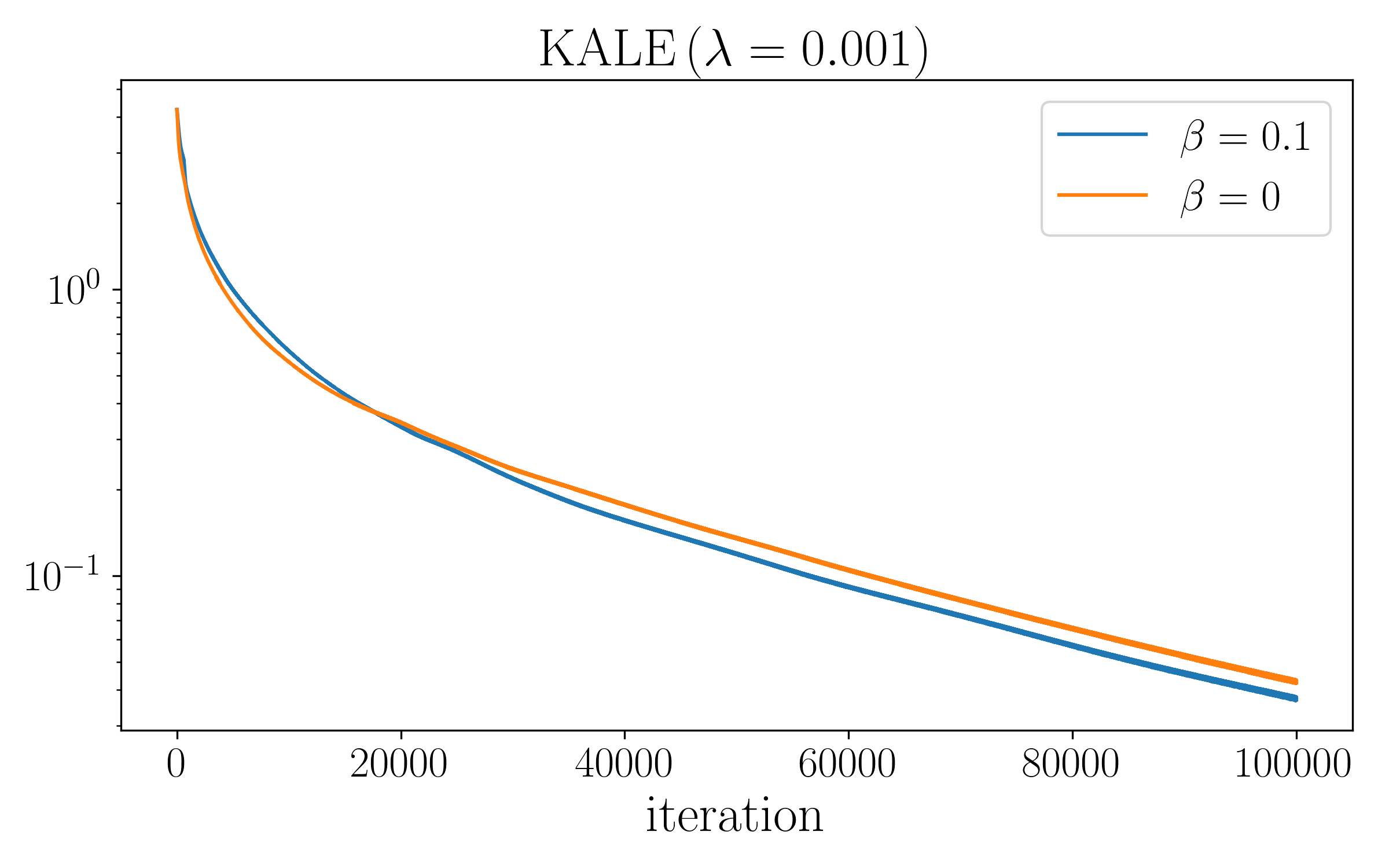} %
    \caption{``Three rings''}

    \label{subfig:noise-injection-ring}
\end{subfigure} % 
\begin{subfigure}[t]{0.49\textwidth}
    \includegraphics[width=\textwidth]{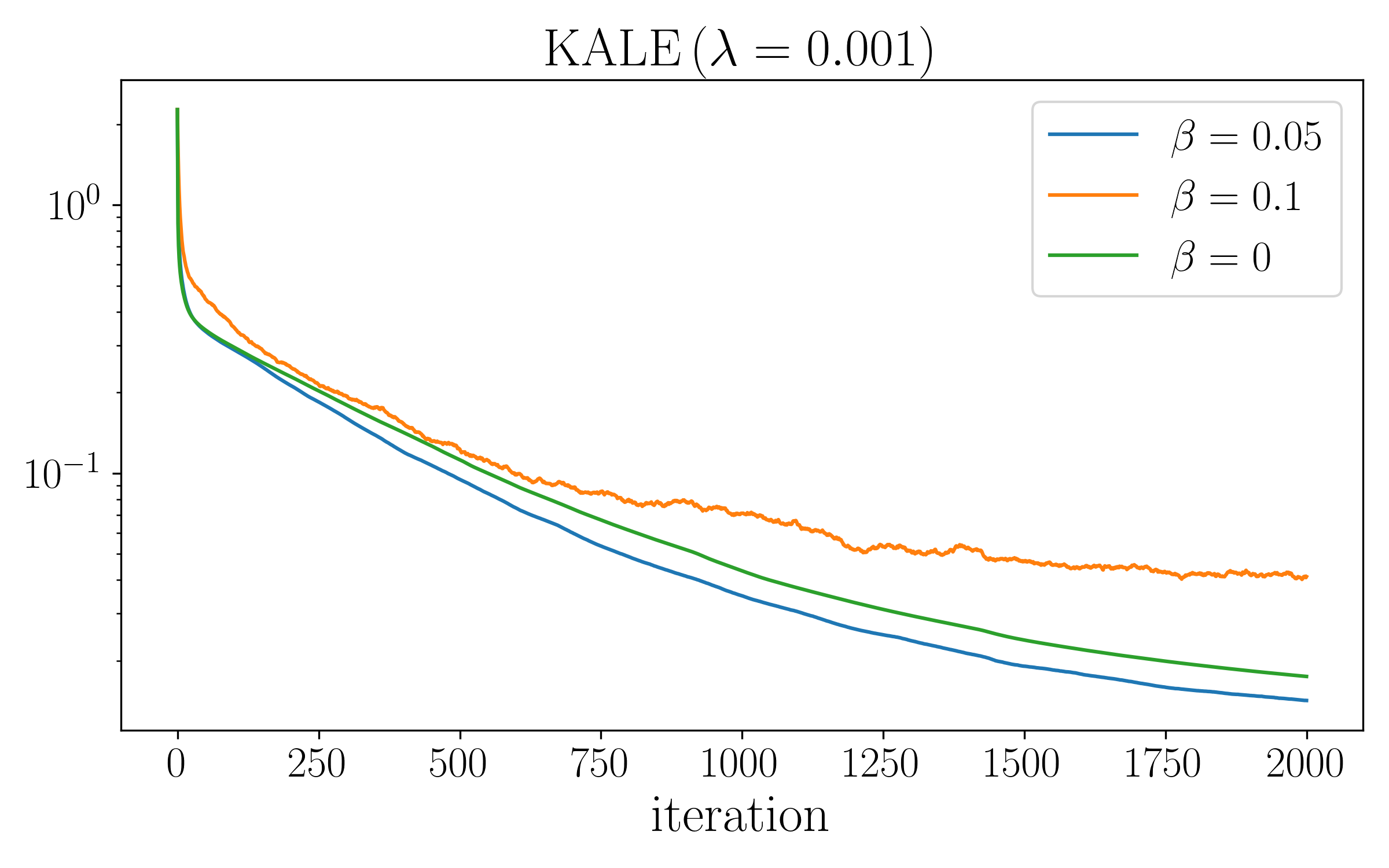} %
    \caption{``Shape transfer''}

    \label{subfig:noise-injection-shape}
\end{subfigure}

\caption{Impact of noise injection on the KALE value during a KALE particle descent algorithm.}
\label{fig:noise-injection-kale-mog-shape}
\end{figure*}
{\bf ``Shape transfer'' experiments} 
For this experiment, we used artificial data from the same
source as \cite{mroueh2020unbalanced}. We sub-sampled both shapes to $ N=2000 $
points, and used a kernel width of $ \sigma = 0.3 $, as well as $
\lambda = 0.001 $.  Because the number of particles is higher in that case, we
used a coordinate descent algorithm to compute KALE, that has a complexity in $
\mathcal  O(N^2) $.  We show in \cref{subfig:noise-injection-shape} the impact of noise injection with a
constant noise schedule of $ \beta_n = 0.05 $. For this experiment, we also
show empirically that while using a small amount of noise \emph{lowers} the
final KALE value when compared to the unregularized KALE flow, a too large
noise level $ \beta_n = 0.1$ results in a \emph{larger} final KALE value. We
hypothesize that that noise schedule did not respect the assumptions made in
\cref{prop:KALE-noise-injection}.

{\bf ``Mixture of Gaussians'' experiments}
For this experiment, we used $ N=240 $ particles for
each distribution, and a standard deviation of $ 0.25 $ for each target
Gaussian.  We used the Unadjusted Langevin Algorithm
\cite{durmus_non-asymptotic_2016} to simulate a KL gradient flow with step size
$ 0.001 $, and the MMD particle descent algorithm of \cite{arbel_maximum_2019_1}
to simulate a MMD flow with step size $0.001$. For both the MMD and the KALE, we
used the same Gaussian kernel with kernel width $ \sigma = 0.35 $. We show the
impact of noise injection for the KALE flow with a constant noise schedule $
\beta_n = 0.3 $ to regularize KALE flow with $ \lambda = 0.001, 0.1 $ and $
10000 $. 

 \begin{figure}[htbp]
   \hspace*{-3em}\includegraphics[width=1.1\textwidth]{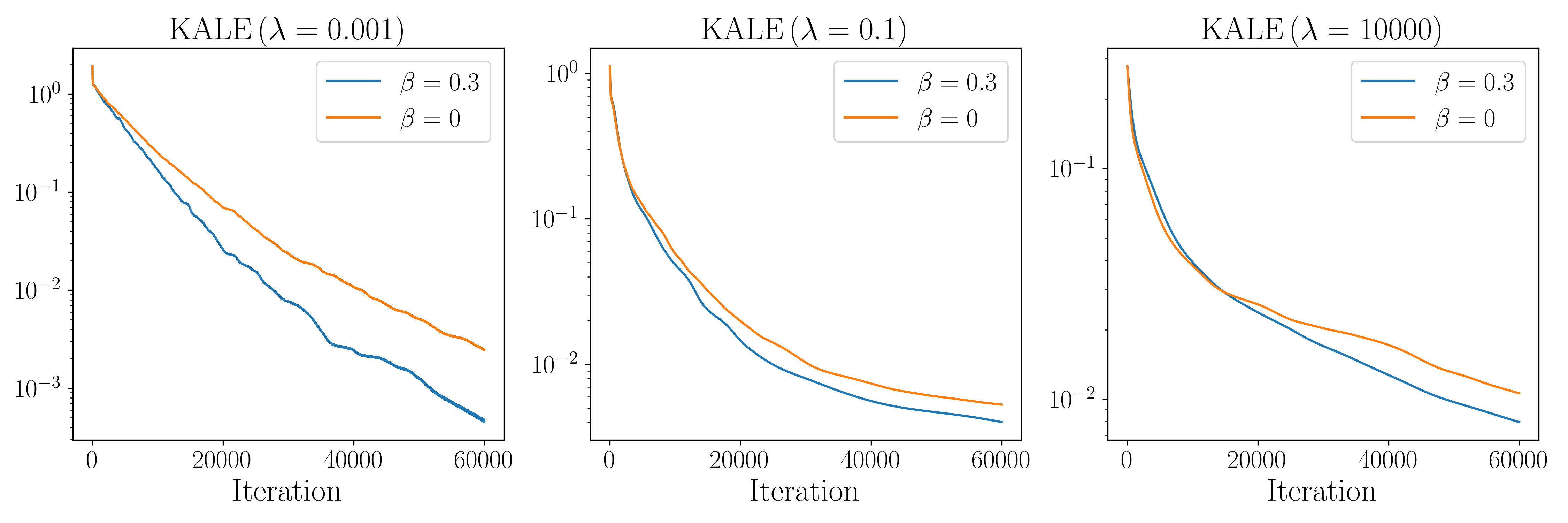} %
    \label{fig:noise-injection-mog}
    \caption{Impact of noise injection: Mixture of Gaussians experiments}
\end{figure}

\end{document}